\documentclass[journal]{IEEEtran}

\IEEEoverridecommandlockouts                              

\usepackage{graphics} 
\usepackage{epsfig} 
\usepackage{times} 
\usepackage{amsmath} 
\usepackage{amssymb}  
\usepackage{amsthm}
\usepackage{algorithm}
\usepackage{algpseudocode}
\usepackage{mathtools}
\usepackage{cite}
\usepackage{caption}
\usepackage{subfig}
\usepackage{balance}

\usepackage{accents}

\DeclareMathOperator*{\argmin}{arg\,min}

\theoremstyle{plain}
\newtheorem{theorem}{Theorem}
\newtheorem*{theorem*}{Theorem}

\newtheorem{proposition}[theorem]{Proposition}
\newtheorem*{proposition*}{Proposition}
\newtheorem{lemma}[theorem]{Lemma}
\newtheorem*{lemma*}{Lemma}

\theoremstyle{definition}
\newtheorem{definition}{Definition}
\newtheorem{assumption}{Assumption}

\theoremstyle{remark}
\newtheorem{case}{Case}

\usepackage{changes}
\definechangesauthor[name={Joseph Norby}, color=orange]{JN}

\makeatletter
\@namedef{Changes@AuthorColor}{red}
\colorlet{Changes@Color}{red}
\makeatother

\title{\LARGE \bf Adaptive Complexity Model Predictive Control
}

\author{Joseph Norby,~\IEEEmembership{Member,~IEEE}, Ardalan Tajbakhsh,~\IEEEmembership{Student Member,~IEEE}, Yanhao Yang,~\IEEEmembership{Student Member,~IEEE}, and Aaron M. Johnson,~\IEEEmembership{Senior Member,~IEEE}
\thanks{*This material is based upon work supported in part by Chevron CTC and the National Science Foundation under Grants DGE-1745016 and ECCS-1924723.}
\thanks{The authors are with the Department of Mechanical Engineering, Carnegie Mellon University, Pittsburgh, PA 15213, USA,
        {\tt\small jcnorby@gmail.com, amj1@andrew.cmu.edu}}%
}

\usepackage[colorlinks,bookmarksopen,bookmarksnumbered,citecolor=black,urlcolor=red]{hyperref}
\newcommand\copyrighttext{%
 \textcopyright 2024 IEEE. Personal use of this material is permitted.
  Permission from IEEE must be obtained for all other uses, in any current or future
  media, including reprinting/republishing this material for advertising or promotional
  purposes, creating new collective works, for resale or redistribution to servers or
  lists, or reuse of any copyrighted component of this work in other works.
  DOI: \href{https://doi.org/10.1109/TRO.2024.3410408}{10.1109/TRO.2024.3410408}}
\newcommand\copyrightnotice{%
\begin{tikzpicture}[remember picture,overlay]
\node[anchor=south,yshift=5pt] at (current page.south) {\fbox{\parbox{\dimexpr\textwidth-\fboxsep-\fboxrule\relax}{  \footnotesize \copyrighttext}}};
\end{tikzpicture}%
}

\begin{document}
\bstctlcite{IEEEexample:BSTcontrol}

\allowdisplaybreaks

\thispagestyle{empty}
\setcounter{page}{0}
\begin{figure*}[t!]
\centering
\large
This paper has been published in IEEE Transactions on Robotics.\\

DOI: \href{https://doi.org/10.1109/TRO.2024.3410408}{10.1109/TRO.2024.3410408}\\

IEEE Explore: \href{https://ieeexplore.ieee.org/document/10551539}{https://ieeexplore.ieee.org/document/10551539}\\

~\\

Please cite the paper as:\\

Joseph Norby, Ardalan Tajbakhsh, Yanhao Yang, and Aaron M. Johnson. ``Adaptive Complexity Model Predictive Control.'' \emph{IEEE Transactions on Robotics}, 40: 4615-4634. 2024.\\

~\\

~\\

\copyrighttext
\vspace{400px}
\end{figure*}
\maketitle
\copyrightnotice

\maketitle
\thispagestyle{empty}
\pagestyle{empty}

\begin{abstract}
This work introduces a formulation of model predictive control (MPC) which adaptively reasons about the complexity of the model while maintaining feasibility and stability guarantees. Existing approaches often handle computational complexity by shortening prediction horizons or simplifying models, both of which can result in instability. Inspired by related approaches in behavioral economics, motion planning, and biomechanics, our method solves MPC problems with a simple model for dynamics and constraints over regions of the horizon where such a model is feasible and a complex model where it is not. The approach leverages an interleaving of planning and execution to iteratively identify these regions, which can be safely simplified if they satisfy an exact template/anchor relationship. We show that this method does not compromise the stability and feasibility properties of the system, and measure performance in simulation experiments on a quadrupedal robot executing agile behaviors over terrains of interest. We find that this adaptive method enables more agile motion (55\% increase in top speed) and expands the range of executable tasks compared to fixed-complexity implementations.
\end{abstract}


\section{Introduction} \label{sec:ac:intro}

As demand for robotic systems increases in industries like environmental monitoring, industrial inspection, disaster recovery, and material handling \cite{bellicoso2018advances,hutter2017anymal,kolvenbach2020towards}, so too has the need for motion planning and control algorithms that efficiently handle the complexity of their dynamics and constraints. Legged systems in particular are well suited for these applications due to their ability to traverse unstructured terrains with behaviors such as that shown in Fig.~\ref{fig:ac:intro}, yet they are so far largely restricted to conservative behaviors due to this complexity. A common approach to overcome these challenges is to break up the problem into a hierarchy of sub-problems which reason over progressively shorter horizons with increasing model complexity. This hierarchy improves computational efficiency which can be used to detect obstacles further away, react more quickly to disturbances, or reduce energy costs. However, this hierarchy is vulnerable to failures caused by omitting portions of the underlying model, raising a fundamental question of how to balance model fidelity and computational efficiency.

\begin{figure}[t]
  \centering
  \includegraphics[width=1.0\linewidth]{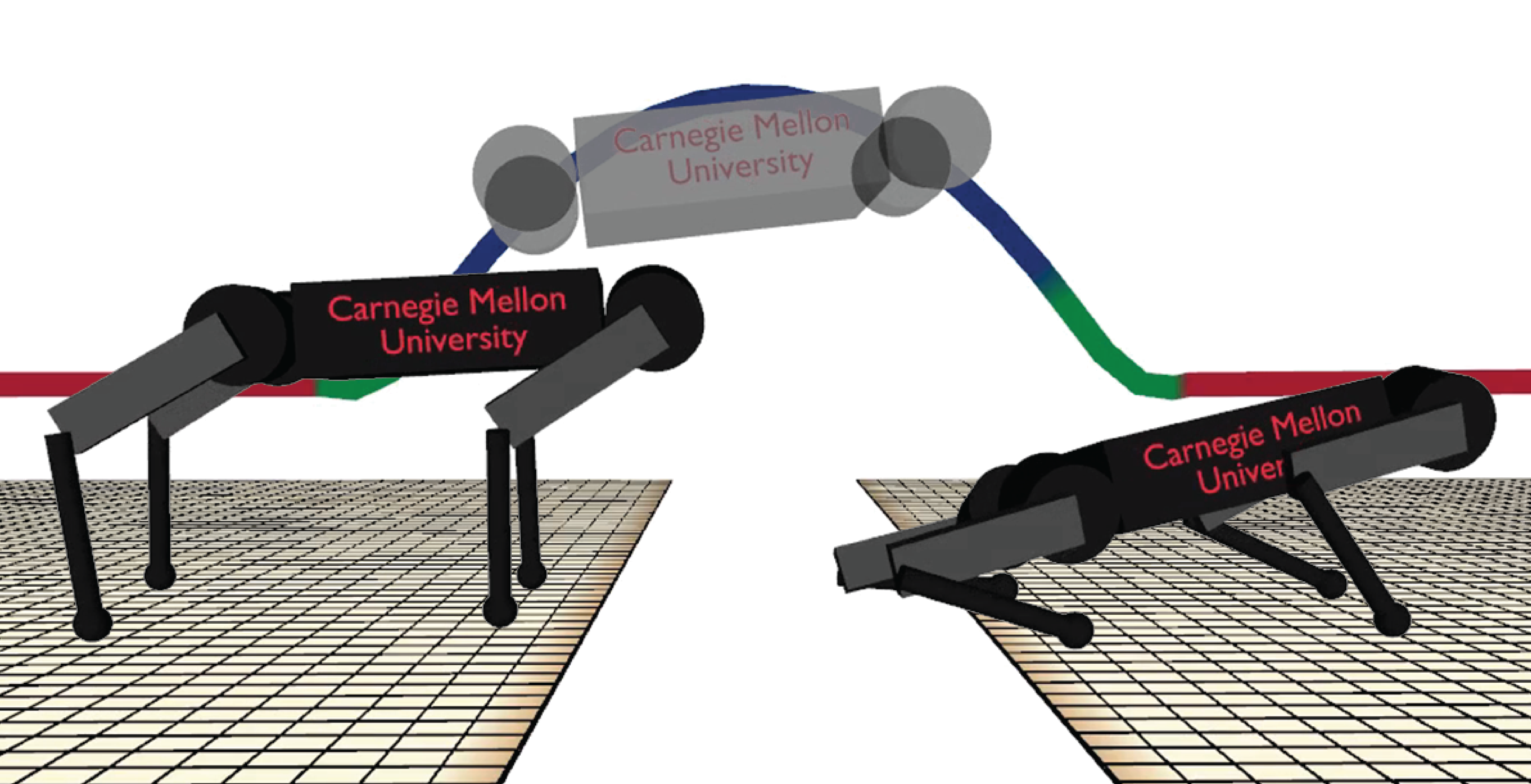}
  \caption[Adaptive complexity model predictive control summary]{Adaptive complexity model predictive control selectively simplifies the model to promote efficiency without sacrificing stability. For example, during a legged leaping task joint information may be required during takeoff and landing but can be omitted elsewhere without affecting the behavior.}
  \label{fig:ac:intro}
\end{figure}

Inspiration for answering this question can be drawn from other scientific fields, in particular behavioral economics and neuromechanics. The famous ``Thinking: Fast and Slow'' framework theorizes that human cognitive function can be described with two systems which respectively handle rapid, simple processing and slow, deliberative reasoning such that ``[the complex, slow system] is activated when an event is detected that violates the model of the world that [the simple, fast system] maintains'' \cite{kahneman2011thinking}. Extending this concept to the field of motion planning yields meta-planning methods which change their structure to leverage simple, fast models where possible and complex, slow ones where the simple model is inaccurate \cite{gochev2011path,fridovich2018planning}. However, it is not well understood under what exact conditions a given dynamical system may leverage a simple model without sacrificing stability and feasibility, nor is it clear when a more complex model should be used without making the computational overhead intractable. 

To investigate these questions, we employ the templates and anchors approach for analyzing model hierarchies \cite{full1999templates}. This framework describes relations between dynamical systems where complex behaviors (anchors) can be captured by simpler models (templates), and its connections to legged locomotion are well-studied in both animals and robots. From these observations we draw three hypotheses:

\begin{enumerate}
    \item From the templates and anchors relationship we can derive sufficient conditions which identify regions of a behavior where complex dynamics and constraints can be safely simplified without compromising stability and feasibility (i.e.\ regions where the robot is following the template model). \label{hyp:ac:conditions}
    \item Legged systems often operate in environments where these conditions are satisfied during the majority of behaviors and therefore stability and feasibility guarantees can often be retained even with simplified models. \label{hyp:ac:simple_frequency}
    \item During behaviors in which these conditions are not met, a controller which leverages adaptive complexity online can improve performance over fixed-complexity formulations by enabling more efficient motion planning while retaining stability guarantees.
    \label{hyp:ac:controller_performance}
\end{enumerate}

We evaluate these hypotheses by constructing a formulation of model predictive control (MPC) that actively adapts the model complexity to the task. This is achieved by iteratively identifying regions of the horizon where the behavior can be safely expressed with a simpler model. This formulation -- which we call adaptive complexity model predictive control (ACMPC) -- is most efficient when the two models satisfy a relationship known as ``exact anchoring'' \cite{libby2016comparative}, and in the worst case converges to the standard all-complex MPC formulation. We evaluate Hypothesis~\ref{hyp:ac:conditions} by showing that this algorithm provides formal stability and feasibility properties with respect to the complex system. We evaluate Hypotheses~\ref{hyp:ac:simple_frequency} and \ref{hyp:ac:controller_performance}  by applying this algorithm to a quadrupedal system and conducting simulation experiments on common environments which legged platforms may encounter. These results show that the majority of the behaviors in these environments admit feasible simplifications and that the resulting improvement in computational efficiency enables a 55\% increase in top speed compared to a system without these reductions. We also show that retaining knowledge of the dynamics and constraints in the complex system expands the range of executable tasks compared to a system which uniformly applies these reductions. In the particular case of legged leaping, this approach enables receding horizon execution of a body-length leaping behavior while considering joint constraints, which expands the leaping ability over prior methods which do not consider these constraints.

The organization of this paper is as follows: Section~\ref{sec:ac:related_work} covers related work in greater detail, and Section~\ref{sec:ac:preliminaries} formulates the problem and introduces notation. Section~\ref{sec:ac:ac_mpc} provides an overview of the algorithm, whose formal properties are explored in Section~\ref{sec:ac:theoretical_analysis}. This algorithm is applied to legged locomotion models in Section~\ref{sec:ac:application_to_legged_systems}, and performance for these systems is quantified in Section~\ref{sec:ac:experiments}. Section~\ref{sec:ac:conclusion} concludes by discussing limitations and future extensions enabled by this work.

\section{Related Work} \label{sec:ac:related_work}

Dynamic motion planning and control for systems with intermittent contact is inherently difficult. However, enabling agile autonomy for such systems is critical for real-world applications that require the robot to touch the world. In particular, legged robots have significant potential for real-world deployment. But, they often suffer from difficulties arising from hybrid dynamics, high state dimensionality, and non-convex constraints on their kinematics and dynamics. Such problems render even basic motion planning problems PSPACE-complete \cite{reif1979complexity}, and as a result existing algorithms to solve them globally for dynamic legged systems cannot operate in real-time \cite{hauser2008motion,mombaur2009using,posa2014direct,dai2014whole}.

In general, the most common approach used to address the planning and control challenges of legged locomotion has been through leveraging some kind of model reduction, wherein the problem is solved with a reduced subset of the state and dynamics. The solution of this simpler problem is then passed to another system with a more complex model which fills in the additional details. The hierarchical nature of this approach enables optimization of each algorithm based on salient features of the problem such as the timescales of the dynamics or rates of sensor information. Examples of this approach include efficient global planners focused on exploration \cite{bartoszyk2017terrain,fernbach2017kinodynamic,tonneau2018efficient,norby2020fast} to local planners that plan contact phases over a few gait cycles \cite{winkler2018gait} to whole-body controllers with full-order representations over very small horizons \cite{sentis2006whole,kuindersma2016optimization,neunert2018whole}. While these hierarchies have been shown to be capable of producing dynamic locomotion, they face a fundamental problem in that dynamics or constraints in the omitted space can render the desired motion sub-optimal or even infeasible. Conservative assumptions in the simple model may fail to produce solutions in difficult environments, and optimistic assumptions may lead to infeasible behaviors that the more myopic complex system may not recognize until too late.

Several methods have been explored to resolve the interface between these layers. One straightforward approach is to use the complex model to assess the simple solution, either by providing a boolean feasible/infeasible classification, computing some value function, or indicating new search directions \cite{zucker2011optimization,plaku2010motion,mcconachie2020learning, carpentier2017learning}. This can be efficient since checking a solution in a higher-order space is easier than searching for one, although it does not allow the simple model to directly reason about the dynamics or constraints in the complex space.

Another promising way to resolve this is to employ an adaptive planning framework to reason over different models based on the task and constraints. One flavor of this approach plans over a mixture of models of varying degrees of fidelity with some pre-defined rules that guide when to use each and how to transition between them \cite{kapadia2013multi,brandao2019multi,norby2020fast}. Other approaches leverage an adaptive composition of models with different safety bounds to trade between performance and robustness or expand the problem dimensionality as needed to find collision-free paths.  \cite{gochev2011path,zhang2012combining,styler2017plan, fridovich2018planning, dornbush2018single}. Our approach is more similar to the latter in terms of the underlying adaptation mechanism, but differs in that we derive the exact conditions under which transitioning between models of varying fidelity can be done without sacrificing stability and feasibility. Furthermore, we demonstrate how such a mechanism can be applied to a receding horizon framework for online planning and control. 

Similar planning and control problems can also be solved online using receding horizon methods. In particular, model predictive control (MPC) is an iterative receding-horizon optimization framework that has been commonly used to solve constrained optimal control problems \cite{allgower2012nonlinear}. In the context of dynamic legged locomotion, MPC often computes feasible body and/or joint trajectories in order to track a higher level reference plan while respecting dynamic, state, and control constraints. Works such as \cite{di2018dynamic, laurenzi2018quadrupedal, shi2019model} compute the desired ground reaction forces using a single rigid body model, which are realized as joint torques using a whole-body controller. Although computationally efficient, this approach typically uses a single simplified model with some simple heuristics to approximate the motion of the legs.
In comparison, this work blends a single rigid body dynamics model with a full kinematics model, allowing us to more exactly capture leg motion via joint limits, collision constraints, and motor models. More generally, our approach is not tied to a particular choice of models, but rather is rather a paradigm for selecting between related models to improve performance.

While simplified model MPC approaches assume a representative reduced-order model (template) exists that fairly approximates the full-order model (anchor), they often do not examine the validity of this approximation. Some approaches have studied safe controller synthesis for the template model while ensuring constraint satisfaction of the anchor model through bounding the differences of the two models using reachability analysis \cite{liu2020leveraging}, approximate simulation properties \cite{kurtz2019formal}, or learning any unmodeled differences \cite{pandala2022robust}. Another approach uses pre-defined ratios to mix the complex model for the immediate future with the simple model for longer horizon planning, which results in more robust locomotion compared to fixed-complexity formulations \cite{li2021model}. Our approach takes inspiration from these in that it defines the exact conditions under which the higher order model can be simplified without violation of formal guarantees, but critically differs in that we adaptively mix models of varying fidelity within any planning horizon based on local feasibility. This allows us to leverage the fidelity of the complex model when necessary while taking advantage of the computation speed enabled by the simple model for planning longer horizon motions.

Many of these hierarchical planning and control approaches have been shown to effectively perform agile and dynamic motions in simulation and hardware. However, they often struggle in generating and executing motions that require the robot to operate at its kinematic and dynamic limits. These motions are critical in enabling behaviours like stepping and leaping over gaps, stairs, and non-traversable obstacles, which are essential in navigating unstructured terrains. Previous approaches lack longer horizon planning, or do not consider joint kinematics or constraints \cite{johnson2013legged,nguyen2019optimized,kolvenbach2019towards,chignoli2022rapid,ponton2021efficient}. Our approach allows the robot to plan for these agile behaviours in a receding-horizon manner while reasoning about constraints over a significantly long horizon, which is shown to expand leaping capabilities in simulation.

Recently, some MPC approaches have been extended to generate dynamically feasible whole-body motion plans in real-time \cite{grandia2022perceptive,meduri2023biconmp,grandia2019feedback,sleiman2021unified,mastalli2020crocoddyl,mastalli2022agile,dantec2021whole}. Our approach is not meant to be a replacement for this class of solutions, but aims to complement them by providing systematic model reduction that reduces the computational load without compromising feasibility. Prior work applying mixed-complexity schedules to the structure-exploiting DDP-based methods common in these works suggests that the concepts are compatible \cite{li2021model}. This reduction in computational load can be beneficial in two ways. First, it can allow other modules (e.g. perception, navigation, autonomy, etc.) to use more resources as needed without exceeding the total computational requirements of the system and lowers the total power draw of onboard computation. Second, it could enable MPC to deploy longer prediction horizons and therefore better reason about future constraints given a fixed computational budget.

\section{Preliminaries} \label{sec:ac:preliminaries}

\begin{table}[t]
    \renewcommand{\arraystretch}{1.2}
    \caption{Key symbols used throughout this paper with section or equation number of the introduction marked}
    \centering
    \begin{tabular}{l l}
    \hline \hline
        \textbf{Core} & \textbf{Sections~\ref{sec:ac:preliminaries} -- \ref{sec:ac:ac_mpc} }\\
    \hline
        $\alpha_{L,I}, \alpha_{U,I}$ & Lower and upper bounds for intermediate cost (\ref{sec:ac:preliminaries}) \\
        $\alpha_{L,T}, \alpha_{U,T}$ & Lower and upper bounds for terminal cost (\ref{sec:ac:preliminaries}) \\
        $\text{ACOCP}$ & Adaptive complexity optimal control problem \eqref{eq:ac:adaptive_ocp} \\
        $f$ & System dynamics \eqref{eq:ac:original_system} \\
        $f_{\textnormal{mpc}}, f_{\textnormal{acmpc}}$ & Closed-loop dynamics under $u_{\textnormal{mpc}}$, $u_{\textnormal{acmpc}}$ \eqref{eq:ac:closed_loop_system}, \eqref{eq:ac:adaptive_system} \\
        $\mathcal{K}, \mathcal{K}_{\infty}$ & Strictly increasing functions (\ref{sec:ac:preliminaries}) \\
        $N$ & Horizon length \eqref{eq:ac:closed_loop_system} \\
        $\text{OCP}$ & Optimal control problem \eqref{eq:ac:ocp} \\
        $\psi_{x}, \psi_{u}, \psi$ & Reductions for states, controls, pairs (\ref{sec:ac:system_definitions}) \\ 
        $\psi^{\dagger}_{x}, \psi^{\dagger}_{u}, \psi^{\dagger}$ & Lifts for states, controls, pairs (\ref{sec:ac:system_definitions}) \\ 
        $S^a, S^f, S$ & Simplicity sets -- adaptive, fixed, total (\ref{sec:ac:algorithm_overview}, \ref{sec:ac:algorithm_overview:S_convergence}) \\
        $\mathcal{S}, \mathcal{S}_a$ & The set of all, admissible simplicity sets \eqref{eq:ac:admissibility_conditions} \\       
        $u_{\textnormal{mpc}}, u_{\textnormal{acmpc}}$ & Fixed, adaptive complexity MPC policies \eqref{eq:ac:control_law}, \eqref{eq:ac:adaptive_control_law} \\
        $u_{T}$ & Terminal set feedback policy (\ref{sec:ac:preliminaries}, \ref{sec:ac:feasibility}) \\
        $V_{N}$ & Cost function over horizon $N$ \eqref{eq:ac:ocp_cost} \\
        $V_{I}, V_{T}$ & Intermediate and terminal costs \eqref{eq:ac:ocp_cost} \\
        $x, u, z$ & State, control input, state/control pair \eqref{eq:ac:original_system}  \\
        $\mathcal{X}, \mathcal{U}, \mathcal{Z}$ & Manifolds for states, controls, pairs (\ref{sec:ac:preliminaries}) \\
        $\mathcal{X}_{f}, \mathcal{U}_{f}, \mathcal{Z}_{f}$ & Feasible sets of states, controls, pairs (\ref{sec:ac:preliminaries}) \\
        $\mathcal{X}_N, \mathcal{X}_{T}$ & Basin of attraction under $N$, terminal set (\ref{sec:ac:preliminaries}) \\
        $\mathbf{x}, \mathbf{u}, \mathbf{z}$ & Trajectories of states, controls, pairs (\ref{sec:ac:adaptive_system_definition})  \\
        $\mathbf{X}, \mathbf{U}, \mathbf{Z}$ & Spaces of trajectories of states, controls, pairs (\ref{sec:ac:adaptive_system_definition})  \\
        $(\cdot)^{*}$ & Quantities determined by solving an OCP (\ref{sec:ac:preliminaries}) \\
        $(\cdot)^{c,s,a}$ & Complex, simple, or adaptive quantities (\ref{sec:ac:ac_mpc}) \\
        $(\cdot)^{l}$ & Shorthand for lifted variables (i.e. $\psi^{\dagger}(\cdot)$) \eqref{eq:ac:adaptive_state_lift}\\
        $(\cdot)_{k,i}$ & Time w.r.t.~dynamical system, prediction horizon \eqref{eq:ac:original_system} \\
    \hline \hline
    \textbf{Analysis} & \textbf{Section~\ref{sec:ac:theoretical_analysis} }\\
    \hline
        $\phi_{f}$ & Rollout function for system $f$ (\ref{sec:ac:constraint_satisfaction}) \\
        $x_{T}$ & Terminal state of OCP solution (\ref{sec:ac:feasibility}) \\
        $\Tilde{\mathbf{x}}, \Tilde{\mathbf{u}}$ & State and control trajectories at successor state (\ref{sec:ac:feasibility}) \\
    \hline \hline
    \textbf{Applications} & \textbf{Section~\ref{sec:ac:application_to_legged_systems} }\\
    \hline
        $\omega, \hat{\omega}$ & Angular velocity vector, matrix \eqref{eq:ac:complex_system_state} \\
        $\theta, \tau$ & Joint positions, torques \eqref{eq:ac:slack_variables} \\
        $D$ & Contact selection matrix \eqref{eq:ac:legged_complex_constraints} \\
        $FK$ & Forward kinematics function \eqref{eq:ac:slack_variables} \\
        $\mathcal{FC}$ & Friction cone feasible set \eqref{eq:ac:legged_complex_constraints} \\
        $h$ & Terrain clearance function \eqref{eq:ac:legged_complex_constraints} \\
        $J$ & Leg Jacobian \eqref{eq:ac:slack_variables} \\
        $J_\omega$ & Map from angular velocity to Euler rates \eqref{eq:ac:legged_complex_dynamics} \\
        $M$ & Inertia matrix in body frame \eqref{eq:ac:ang_acc_map} \\
        $n, n_j$ & Number of limbs, joints per limb (\ref{sec:ac:application_to_legged_systems}) \\
        $Q_{I}, Q_{T}$ & State cost matrices -- intermediate and terminal \eqref{eq:ac:complex_system_state} \\
        $q_\textnormal{lin}, q_\textnormal{ang}, q_\textnormal{foot}$ & Body position, body orientation, foot position \eqref{eq:ac:complex_system_state} \\
        $R$ & Control cost matrix \eqref{eq:ac:complex_system_state} \\
        $R(q_\textnormal{ang})$ & Rotation matrix between body and world frames \eqref{eq:ac:ang_acc_map} \\
        $u_\textnormal{body}, u_\textnormal{foot}$ & Forces applied to body, feet \eqref{eq:ac:complex_system_state} \\
        $W$ & Shorthand for net angular acceleration map \eqref{eq:ac:ang_acc_map} \\
        $\bar{\mathbf{x}}, \bar{\mathbf{u}}$ & State and control references \eqref{eq:ac:legged_complex_cost}\\
        $(\cdot)_{\textnormal{min},\textnormal{max}}$ & Bounds for a given quantity \eqref{eq:ac:legged_complex_constraints} \\
    \hline \hline
    \end{tabular}
    \label{tab:ac:notation}
\end{table}

To clarify the operation of adaptive complexity MPC and its properties, we define a formulation for model predictive control and the closed-loop system it yields following \cite{rawlings2017model} and using notation described in Table~\ref{tab:ac:notation}. Consider a nonlinear, discrete-time, dynamical system which evolves on state manifold $\mathcal{X}$ under controls $\mathcal{U}$ and with dynamics $f$,
\begin{align} \label{eq:ac:original_system}
    x_{k+1} &= f(x_{k}, u_{k})
\end{align}
where $x_{k} \in \mathcal{X}, u_{k} \in \mathcal{U}$ are the current state and control at time $k  \in \mathbb{Z}$, and $x_{k+1} \in \mathcal{X}$ is the successor state. Let $\mathcal{X}_{f} \subseteq \mathcal{X}$ be the set of all feasible states and $\mathcal{U}_f \subseteq \mathcal{U}$ the set of feasible controls. For the sake of notational simplicity, let $z_{k} \coloneqq (x_{k},u_{k})$ define a state-control pair such that $x_{k+1} = f(z_{k})$. Let the space of all such pairs be $\mathcal{Z}=\mathcal{X} \times \mathcal{U}$ and the space of all feasible pairs be $\mathcal{Z}_{f} = \mathcal{X}_{f} \times \mathcal{U}_{f}$.

Let the set $\mathcal{X}_N$ denote the basin of attraction of the controller parameterized by $N$. We list a few standard assumptions on the system in \eqref{eq:ac:original_system} and the set $\mathcal{X}_N$:

\begin{assumption} \label{as:ac:system_conditions}
(A) $f(0,0) = 0$ (the origin is an equilibrium point). \\
(B) $\exists u \in \mathcal{U}_{f} \mid f(x,u) \in \mathcal{X}_N~\forall x \in \mathcal{X}_N$ ($\mathcal{X}_N$ is control positive invariant).\\
(C) $\mathcal{X}_{f}$ and $\mathcal{U}_{f}$ are compact and contain the origin in their interiors.
\end{assumption}

To formalize model predictive control, first define a predicted control trajectory with horizon $N$ and starting at time $k$ as $\mathbf{u}_{k} = [u_{0 \mid k}, u_{1 \mid k}, \dots, u_{i \mid k}, \dots, u_{N-1 \mid k} ]$ where $i$ denotes the index within the horizon. In many cases the subscript $k$ can be inferred from context, allowing the simpler form $\mathbf{u} = [u_{0 }, u_{1}, \dots, u_{i}, \dots, u_{N-1} ]$ . Likewise let $\mathbf{x} = [x_0, x_1, \dots, x_{i}, \dots, x_{N}  ]$ denote the predicted state trajectory, and their pairing as $\mathbf{z} = [z_0, z_1, \dots, z_{i}, \dots, z_{N-1} ]$ such that $z_i = (x_i, u_i )$. Let the spaces of all state, control, and paired trajectories be denoted as $\mathbf{X}$, $\mathbf{U}$, and $\mathbf{Z}$ respectively. The optimal control problem (OCP) solved in the standard NMPC formulation with terminal cost and region is thus $\text{OCP}_{N}(x_{k}) : \mathcal{X}_{f} \rightarrow \mathbf{Z}$,
\begin{subequations}
\label{eq:ac:ocp}
\begin{align}
    \mathbf{z}^*_{k} &= \text{OCP}_{N}(x_{k}) \nonumber \\
    &= \argmin_{\mathbf{z}_{k}} \left(V_{T}(x_{N}) + \sum_{i = 0}^{N - 1} V_{I}(z_{i})\right) \label{eq:ac:ocp_cost} \\
    \text{s.t.} \quad x_{i + 1} &= f(z_i) ~ \forall i = 0, \dots, N - 1 \label{eq:ac:ocp_dynamics} \\
    z_i &\in \mathcal{Z}_{f} ~ \forall i = 0, \dots, N - 1 \label{eq:ac:ocp_constraint} \\
    x_0 &= x_k \label{eq:ac:ocp_initial_state}\\
    x_N &\in \mathcal{X}_{T} \label{eq:ac:ocp_terminal_state}
\end{align}
\end{subequations}
where $\mathbf{z}^*_{k}$ contains the optimal state and control trajectories $\mathbf{x}^*_{k}$ and $\mathbf{u}^*_{k}$, $V_{I}(\cdot)$ and $ V_{T}(\cdot)$ are the intermediate and terminal costs, $\mathcal{Z}_{f}$ is the set of feasible state-control pairs, and $\mathcal{X}_{T}$ is a terminal set. In most contexts the parameterization by $N$ is implied and thus the problem will be denoted as simply $\text{OCP}(x_k)$. The control law defined by NMPC is determined by solving (\ref{eq:ac:ocp}) at each time $k$ and applying the first control, i.e.\ $u_k = u^*_{0 \mid k}$. This defines the state feedback policy $u_{\textnormal{mpc}}(x_{k})$ and resulting closed loop system $f_{\textnormal{mpc}}(x_{k})$,
\begin{align}
    u_{\textnormal{mpc}}(x_{k}) &\coloneqq u^*_{0 \mid k} \label{eq:ac:control_law} \\
    f_{\textnormal{mpc}}(x_{k}) &\coloneqq f(x_k,u_{\textnormal{mpc}}(x_{k})) \label{eq:ac:closed_loop_system}
\end{align}
Standard stability proofs for NMPC formulations typically rely on showing that the closed-loop system admits a Lyapunov function which is upper and lower bounded by strictly increasing functions of state and is strictly decreasing in time (for asymptotic stability) or bounded in magnitude by the control input (for Input-to-State Stability). We borrow the standard definitions of strictly increasing functions $\mathcal{K}$ and $\mathcal{K}_\infty$ as well as Lyapunov functions and asymptotic stability from \cite{limon2009input}. 

For the closed-loop system in (\ref{eq:ac:closed_loop_system}) to yield provable stability, assume the following properties on $V_{I}(\cdot), V_{T}(\cdot),$ and $\mathcal{X}_{T}$,

\begin{assumption} \label{as:ac:original_mpc_conditions}
(A) Intermediate and terminal cost are upper and lower bounded by $\mathcal{K}_\infty$ functions, i.e.~$\exists \alpha_{L,I}, \alpha_{U,I}, \alpha_{L,T}, \alpha_{U,T} \in \mathcal{K}_\infty$ such that:
\begin{align*}
    \alpha_{L,I}(|x|) \leq & V_{I}(x,u) \leq \alpha_{U,I}(|x|) && \forall x \in \mathcal{X}_{f}, u \in \mathcal{U}_{f} \\
    \alpha_{L,T}(|x|) \leq & V_{T}(x) \leq \alpha_{U,T}(|x|) && \forall x \in \mathcal{X}_{T}.
\end{align*}
(B) A solution to (\ref{eq:ac:ocp}) exists for all $x_k \in \mathcal{X}_N$. \\
(C) The functions $V_{I}(x,u), V_{T}(x),$ and $f$ are all twice differentiable with respect to $x$ and $u$. \\
(D) $\exists \alpha_{V_{T}} \in \mathcal{K_\infty}, u_{T}(x) : \mathcal{X}_{T} \rightarrow \mathcal{U}_f$ such that
\begin{align*}
    f(x,u_T(x)) &\in \mathcal{X}_T \\
    V_{T}(f(x,u_{T}(x))) - V_{T}(x) &\leq -\alpha_{V_{T}}(|x|)
\end{align*}
$\forall x \in \mathcal{X}_{T}$ (terminal control law decreases cost). \\
(E) The control law defined in (\ref{eq:ac:control_law}) satisfies Assumption \ref{as:ac:system_conditions}A.
\end{assumption}

It is a known result that under these conditions, the origin of the system defined in \eqref{eq:ac:closed_loop_system} is asymptotically stable for all $x_0 \in \mathcal{X}_N$ \cite{allgower2012nonlinear,limon2009input}. These results can be extended to show the stability of other stationary points through a coordinate transformation or the stability of a time-varying reference by transforming to an equivalent time-varying system \cite{rawlings2017model}.

\section{Adaptive Complexity MPC} \label{sec:ac:ac_mpc}

\subsection{Algorithm Overview}\label{sec:ac:algorithm_overview}
The core idea of adaptive complexity is to leverage models of differing complexity to simplify the model in regions where feasibility is assured and only increase complexity as needed to guarantee feasibility and stability properties. Our approach is to recursively define a simplicity set $S_k$ at each time $k$ whose elements are times $i$ in the prediction horizon where state-control pairs can be simplified. This simplicity set is then used to define a smaller OCP which can be solved more efficiently and from which a standard MPC control policy can be extracted. The OCP solution is then used to update the next simplicity set, $S_{k+1}$.

Conditions on this set $S_k$ can be directly drawn from the literature on templates and anchors \cite{full1999templates} --  elements of $S_k$ represent times corresponding to state-control pairs that follow a feasible, attracting, invariant submanifold within the complex (``anchor'') space, i.e.\ follow the ``template'' dynamics. In other words, membership in $S_k$ implies that the system remains on the manifold after applying the complex dynamics and without violating constraints, as shown visually in Fig.~\ref{fig:ac:S_illustration}. This approach relaxes the typically restrictive requirement that template and anchor relations are \emph{always} satisfied for two dynamical systems, and replaces it with a method to identify \emph{when} they can be satisfied at a given time. This allows the algorithm to account for differences between the two systems, while improving performance proportionally to the extent to which they align. Note that while this work assumes one template per anchor, it could be further extended to include multiple templates describing different reductions of the anchor dynamics.

\begin{figure}[t]
  \centering
  \includegraphics[width=1.0\linewidth]{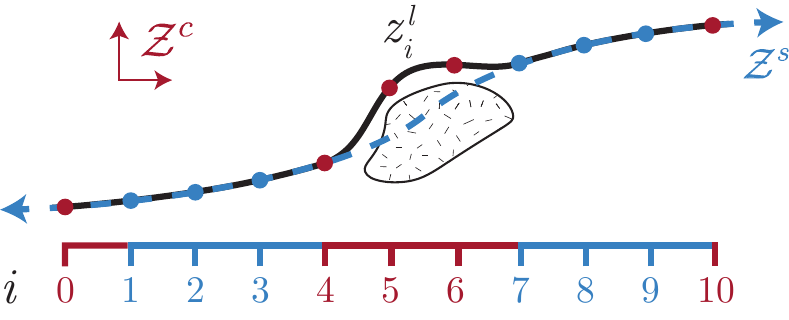}
  \caption[Simplicity set illustration]{Elements in the horizon $i$ are in the set $S_k$ if they are feasible and stay on the manifold $\mathcal{Z}^s$ (illustrated as a blue dashed 1D curve). The adaptive system allows $z^l_{i}$ to leave this submanifold while remaining in the manifold $\mathcal{Z}^c$ (the surrounding white 2D space). Elements in the set $S_k$ are denoted in blue, elements not in this set are denoted in red. In this example, $S_k = \{1, 2, 3, 7, 8, 9\}$.}
  \label{fig:ac:S_illustration}
\end{figure}

\subsection{Complex and Simple System Definitions} \label{sec:ac:system_definitions}
Let the original system defined in (\ref{eq:ac:original_system}) be ``complex'' which is clarified by the superscript $(\cdot)^c$. Let the ``simple'' system be denoted with the superscript $(\cdot)^s$ such that the state $x^s$ lies on the manifold $\mathcal{X}^s$, where $\dim \mathcal{X}^s \leq\dim \mathcal{X}^c$. These states are related by the state reduction $\psi_x : \mathcal{X}^c \rightarrow \mathcal{X}^s$ defined as $x^s = \psi_x(x^c)$. Examples of these systems (which are further explored in Section~\ref{sec:ac:application_to_legged_systems}) could be the full configuration space of a legged robot or its centroidal momenta as $\mathcal{X}^c$, and some related system such as the single rigid body or inverted pendulum models as $\mathcal{X}^s$ \cite{wensing2022optimization}, where $\psi_{x}$ is the projection -- often via a selection matrix -- from the complex space to the simple one. Note that the two systems could also share the same state space but differ in feasible sets, for example if one system has a more complex representation of constraints (motivating our choice of the terms ``complex/simple'' over ``full/reduced order'').

Let each system have controls $u^c$ and $u^s$ defined over manifolds $\mathcal{U}^c$ and $\mathcal{U}^s$ which are related by the control reduction $\psi_u : \mathcal{U}^c \times \mathcal{X}^c \rightarrow \mathcal{U}^s$ defined as $u^s = \psi_u(u^c, x^c)$. Define the state-control pairs as $z^c = (x^c, u^c)$ and $z^s = (x^s, u^s)$ which lie on manifolds $\mathcal{Z}^c \coloneqq \mathcal{X}^c \times \mathcal{U}^c$ and $\mathcal{Z}^s \coloneqq \mathcal{X}^s \times \mathcal{U}^s$. This permits the definition of the reduction $\psi: \mathcal{Z}^c \rightarrow \mathcal{Z}^s$ defined as $z^s = \psi(z^c) = (\psi_x(x^c), \psi_u(u^c, x^c) )$. Examples of varying fidelity control representations include actuator inputs for $\mathcal{U}^c$ and ground reaction forces, impulses, or impedances for $\mathcal{U}^s$, where $\psi_u$ again projects from the complex input to an equivalent simple one.

While the many-to-one reduction $\psi$ defines how the complex and simple systems relate and is generally determined by the choices of these two systems (e.g.~a projection), more flexibility is allowed in determining how to lift from the simple to the complex space. Define the mapping $\psi^\dagger$ such that $\psi \circ \psi^\dagger = I$, where $I$ is the identity map. Let $\psi^\dagger_x$ and $\psi^\dagger_u$ give the outputs of $\psi^\dagger$ corresponding to state and control variables, respectively. Examples of $\psi^{\dagger}_{x}$ for legged systems include inverse kinematics algorithms which map a body state and reference foot trajectories to joint kinematics, while examples of $\psi^{\dagger}_{u}$ include controllers which map inputs from task space to actuator space.

The dynamics and constraints for the complex system have already been defined in Section~\ref{sec:ac:preliminaries}. Define the dynamics $f^s$ and feasible set of the simple system $\mathcal{Z}^s_{f}$,
\begin{align}
    x^{s}_{k+1} &= f^s(z^s_{k}) \label{eq:ac:simple_dynamics} \\ 
    z^s_{k} &\in \mathcal{Z}^s_{f} \label{eq:ac:simple_constraints}.
\end{align}

Let the basin of attraction of the complex system be denoted $\mathcal{X}_{N^c}$ under horizon length $N^c$.

\subsection{Adaptive System Definition} \label{sec:ac:adaptive_system_definition}

We seek an adaptive control law which leverages the simple system when the system can feasibly remain on the manifold $\mathcal{Z}^s_{f}$ and the complex system when it cannot. Define another set of state and control variables $x^a_{i}$, $u^a_{i}$, and $z^a_{i}$ which represent an adaptive system used to solve the OCP. Note that this `system' is only defined for the OCP and thus is always indexed by prediction horizon time $i$ and not the dynamical system time $k$.

We will use the simplicity set $S_k$ to assign these quantities at each time $i$ in the horizon to a particular system. More formally, denote all combinations of horizon elements as $\mathcal{S} \coloneqq \left\{ S \mid S \subseteq \{0, 1, \dots, i, \dots, N\} \right\}$, and let $S_{k} \in \mathcal{S}$. This simplicity set then assigns a domain to each element of the horizon,
\begin{align}
    z^a_{i} \in& \begin{cases}
        \mathcal{Z}^c, & i \notin S_k \\
        \mathcal{Z}^s, & i \in S_k 
    \end{cases} \label{eq:ac:adaptive_state}
\end{align}
Often the adaptive states and controls at time $i$ must be expressed in the complex system whether or not $i \in S_{k}$. These instances can be handled by  applying the lifting operator $\psi^\dagger$ at time $i$ if $i \in S_{k}$. To simplify subsequent notation and clarify the difference between a quantity in the complex system (i.e.~$z^c_{i}$) and a quantity in the adaptive system which may have been mapped to the complex system (i.e.~$\psi^{\dagger}(z^a_{i})$), we will use the shorthand $z^l_{i} \in \mathcal{Z}^c$ to denote the mapped quantity, which is defined by,
\begin{align}
    z^l_{i} &= \begin{cases}
        z^a_{i}, & i \notin S_k \\
        \psi^\dagger(z^a_{i}), & i \in S_k 
    \end{cases} \label{eq:ac:adaptive_state_lift}
\end{align}

Next we define the dynamics of the adaptive system which are used to solve the OCP,
\begin{align}
    x^{a}_{i+1} &= f^a(z^a_{i}) \coloneqq \begin{cases}
        f^c(z^a_{i}) & i, i+1 \notin S_k \\
        \psi_{x} \circ f^c(z^a_{i}) & i \notin S_k,\, i+1 \in S_k \\
        f^c \circ \psi^\dagger(z^a_{i}) & i \in S_k,\, i+1 \notin S_k \\
        f^s(z^a_{i}) & i, i+1 \in S_k
    \end{cases} \label{eq:ac:adaptive_system_dynamics}
\end{align}
where $x^{a}_{i+1}$ is the successor state in the adaptive system. The OCP for the adaptive system uses these dynamics to construct feasible motions over a prediction horizon $N^a \geq N^c$.

Denote a predicted adaptive control sequence over horizon $N^a$ as $\mathbf{u}^a = [u^a_{0}, u^a_{1}, \dots, u^a_{N^a - 1} ]$, a predicted adaptive state sequence as $\mathbf{x}^a = [x^a_{0}, x^a_{1}, \dots, x^a_{N^a} ]$, a predicted adaptive state-control pair sequence as $\mathbf{z}^a = [z^a_{0}, z^a_{1}, \dots, z^a_{N^a - 1} ]$, and their respective spaces as $\mathbf{X}^a, \mathbf{U}^a,$ and $\mathbf{Z}^a$. Let the lifted form of these trajectories be denoted as $\mathbf{z}^l$. We define the constraints $\mathcal{Z}^a_{f,i}$ in the adaptive system,
\begin{align} \label{eq:ac:adaptive_feasible_set}
     \mathcal{Z}^a_{f,i} &\coloneqq
    \begin{cases}
        \mathcal{Z}^c_{f} & i \notin S_k \\
        \mathcal{Z}^s_{f} & i \in S_k 
    \end{cases}.
\end{align}

\subsection{Adaptive Complexity Optimal Control Problem} \label{sec:ac:acocp}

With these definitions in place we now state the adaptive complexity optimal control problem $\text{ACOCP}_{N^a}(x^c_{k}, S_k) : \mathcal{X}^c_f \times \mathcal{S} \rightarrow \mathbf{Z}^a$ which is solved to determine the control input $u^c_k$,
\begin{subequations}
\label{eq:ac:adaptive_ocp}
\begin{align}
    \mathbf{z}^{*a}_{k} &= \text{ACOCP}_{N^a}(x^c_{k}, S_k) \nonumber \\
    &= \argmin_{\mathbf{z}^a_{k}} \left(V_{T}^c(x^a_{N}) + \sum_{i = 0}^{N - 1} V_{I}^a(z^a_{i})\right) \label{eq:ac:adaptive_ocp:cost}\\
    \text{s.t.} \quad x^a_{i + 1} &= f^a(z^a_{i}) ~ \forall i = 0, \dots, N^a - 1 \label{eq:ac:adaptive_ocp:dynamics} \\
    z^a_{i} &\in \mathcal{Z}^a_{f,i} ~ \forall i = 0, \dots, N^a - 1 \label{eq:ac:adaptive_ocp:constraints}\\
    x^a_0 &= x^c_k \label{eq:ac:adaptive_ocp:initial_condition} \\
    x^a_{N^a} &\in \mathcal{X}^c_{T} \label{eq:ac:adaptive_ocp:terminal_constraint}
\end{align}
\end{subequations}
where the adaptive cost function is equal to the complex system cost evaluated on the lifted state-control pair,
\begin{align} \label{eq:ac:adaptive_cost}
V_{I}^a(z^a_{i}) &= \begin{cases}
    V_{I}^c(z^a_{i}) & i \notin S_k \\
    V_{I}^c(\psi^\dagger(z^a_{i})) & i \in S_k 
\end{cases}
\end{align}
In many contexts the ACOCP parameterization by $N^a$ and corresponding simplicity set $S_k$ can be inferred, permitting the shorthand $\text{ACOCP}(x^c_k)$. Let $\mathcal{X}_{N^a}$ be the set of states for which the solution to (\ref{eq:ac:adaptive_ocp}) exists. Let the optimal value function found in (\ref{eq:ac:adaptive_ocp}) correspond to the control trajectory $\mathbf{u}^{*a} = [ u^{*a}_0, u^{*a}_1,\dots,u^{*a}_{N^a-1}]$ and corresponding lifted trajectory $\mathbf{u}^{*l}$. The control law defined by ACMPC is determined by solving (\ref{eq:ac:adaptive_ocp}) at each time $k$ and applying the first control, i.e.\ $u^c_{k} = u^{*l}_{0 \mid k}$. Define the state feedback policy $u^c_{\textnormal{acmpc}}(x^c_{k})$ and resulting closed loop system $f^c_{\textnormal{acmpc}}(x^c_{k})$,
\begin{align}
    u^c_{\textnormal{acmpc}}(x^c_{k}) &\coloneqq u^{*l}_{0 \mid k} \label{eq:ac:adaptive_control_law} \\
    f^c_{\textnormal{acmpc}}(x^c_{k}) &\coloneqq f^c(x^c_{k},u^c_{\textnormal{acmpc}}(x^c_{k})) \label{eq:ac:adaptive_system}
\end{align}
Note that each term in the ACOCP defined in (\ref{eq:ac:adaptive_ocp}) converges to its complex counterpart if $S_k = \{\}$, so the behavior of the original MPC closed-loop system defined in \eqref{eq:ac:original_system} can always be retained. However, we seek to find the minimal complexity required to still ensure stability of the closed-loop system.

\subsection{Conditions on the Complexity Set} \label{sec:ac:algorithm_overview:S_convergence}

The set $S_k$ clearly cannot be arbitrary in order to maintain stability, as ignoring some uncontrolled component of the system dynamics could easily cause undesired behavior. To avoid this we define a notion of the admissibility of $S_k$ which is needed to show that the state and control trajectories in the adaptive space match their realizations in the complex space.
\begin{definition} \label{def:ac:admissibility} (\textbf{Admissibility} of $S_{k}$)
The simplicity set $S_{k}$ defined at time $k$ is admissible if the following conditions hold for all $i \in S_{k}$,
\begin{subequations} \label{eq:ac:admissibility_conditions}
\begin{align}
    i &\neq 0, N^a \label{eq:ac:admissibility_terminal_states} \\
    \psi^\dagger \circ \psi (z^l_{i}) &\in \mathcal{Z}^c_{f} \label{eq:ac:admissibility_feasible} \\ 
    \psi^\dagger_x \circ f^s  \circ \psi(z^l_{i}) &= f^c(z^l_{i}) \label{eq:ac:admissibility_anchor}
    \\
    \psi^\dagger_x \circ \psi \circ f^c  (z^l_{i-1}) &= f^c(z^l_{i-1}). \label{eq:ac:admissibility_enter_manifold}
\end{align}
\end{subequations}
\end{definition}
The condition \eqref{eq:ac:admissibility_terminal_states} requires that the first and last state be complex, which ensures the predicted trajectory matches the actual dynamics and that the system is able to reach the (possibly complex) terminal state. The conditions \eqref{eq:ac:admissibility_feasible} require that the state and control on the manifold at $i$ are feasible in the complex space, \eqref{eq:ac:admissibility_enter_manifold} requires the dynamics from the prior state and control lead to the manifold, and \eqref{eq:ac:admissibility_anchor} requires that the dynamics applied to the current state yields a successor state on the manifold, i.e.\ the complex space ``exactly anchors'' the simple space (in the sense of \cite[Appendix A]{libby2016comparative}). These concepts are illustrated in Fig.~\ref{fig:ac:proof_illustration}. To codify the admissibility of a simplicity set $S_{k}$, define the set of all admissible sets as $\mathcal{S}_a \subset \mathcal{S}$. Membership in $\mathcal{S}_a$ can be determined by computing the lifted prediction trajectory $\mathbf{z}^l$ and checking the conditions in \eqref{eq:ac:admissibility_conditions} for each $i \in S_{k}$. Also note that removing any element from an admissible simplicity set does not invalidate the admissibility property.

In order to show the stability of adaptive complexity MPC, we require that $S_k$ be admissible for all $k$.
\begin{assumption} \label{as:ac:S_admissible}
$S_k \in \mathcal{S}_a ~ \forall k \geq 0$. \\
\end{assumption}

The remaining challenge is ensuring the admissibility of $S_k$, which is inductively handled by Algorithm \ref{alg:ac:acmpc} and for which recursive admissibility is proven in Sec.~\ref{sec:ac:recursive_admissibility},~Lemma \ref{th:ac:recursive_admissibility}. There are multiple methods to ensure the initial $S_0$ is admissible, such as initializing with $S_0 = \{\}$ which is guaranteed to be admissible, or iteratively solving (\ref{eq:ac:adaptive_ocp}) and updating $S_0$ until an admissible set is found. After an initial set is found, successor sets can be found by combining two simplicity sets, $S^a_k$ and $S^f$. The set $S^a_k$ adaptively identifies regions that can be safely simplified, while $S^f$ requires that certain elements always remain complex for stability and admissibility, in particular the first and last element. After a solve, $S^a_k$ is updated by checking \eqref{eq:ac:admissibility_conditions} to determine which states can be simplified. Its elements are then shifted in time, i.e.\ $S^a_{k+1} \gets  \{i \mid (i+1 \in S^a_{k} ) \vee (i = N^a) \}$, and combined with the fixed set to yield the successor simplicity set $S_{k+1} = S^a_{k+1} \cap S^f$. This approach is sufficient to guarantee admissibility in the nominal case with a perfect model, where the only new portion of the optimal trajectory is the last element which is always covered by $S^f$. Robustness can be improved at the expense of computational effort by expanding $S^f$ to include more elements to ensure $S_k$ remains admissible under disturbances. This is similar to the MPC formulation in \cite{li2021model}, but includes the adaptive term $S^a_k$ to ensure feasibility across the entire horizon.

\subsection{Formal Definition of Adaptive Complexity MPC Algorithm}

With the control law and conditions for admissibility we can now summarize adaptive complexity MPC as an iterative algorithm shown in Algorithm~\ref{alg:ac:acmpc}. This procedure combines the fixed and adaptive simplicity sets, solves the ACOCP, extracts the lifted solution, updates the adaptive simplicity set from this solution, then applies the first element of the control trajectory. Note that this algorithm is no different from standard MPC formulations with the exception of the definition of the simplicity set and lifting of the resulting trajectory.

\begin{algorithm}[t]
\caption{Adaptive Complexity Model Predictive Control}\label{alg:ac:acmpc}
\begin{algorithmic}
\State \textbf{Given} $x^c_0, S_0, S^f, N^a$
\State $k \gets 0$
\State $S^a_0 \gets S_0$
\Repeat
\State $S_k \gets S^a_k \cap S^f$
\State $\mathbf{z}^{*a} \gets \text{ACOCP}_{N^a}(x^c_{k}, S_k)$ \Comment{(\ref{eq:ac:adaptive_ocp})}
\State $\mathbf{z}^{*l} \gets \psi^\dagger(\mathbf{z}^{*a})$ \Comment{(\ref{eq:ac:adaptive_state_lift})}
\State $u^c_{k} \gets u^{*l}_{0 \mid k}$ \Comment{(\ref{eq:ac:adaptive_control_law})}
\State $x^c_{k+1} \gets f^c(x^c_{k},u^c_{k})$ \Comment{(\ref{eq:ac:adaptive_system})}
\State $S^a_{k} \gets \{i \mid i \in S^a_{k} \wedge z^l_{i} \text{ satisfies } (\ref{eq:ac:admissibility_conditions}) \}$
\State $S^a_{k+1} \gets  \{i \mid (i+1 \in S^a_{k} ) \vee (i = N^a) \}$
\State $k \gets k + 1$
\Until {finished}
\end{algorithmic}
\end{algorithm}

\section{Theoretical Analysis} \label{sec:ac:theoretical_analysis}

This section describes the theoretical properties of adaptive complexity MPC. We first show that constraints of the original OCP in (\ref{eq:ac:ocp}) are satisfied by solutions of the ACOCP in (\ref{eq:ac:adaptive_ocp}) under assumptions on the admissibility of $S_k$ (Sec.~\ref{sec:ac:constraint_satisfaction}). We use this result to show recursive feasibility of the ACOCP and thus asymptotic stability of a point of interest of the closed loop system (Sec.~\ref{sec:ac:feasibility}) under some assumptions on the form of $\psi^\dagger$. We show that Algorithm~\ref{alg:ac:acmpc} satisfies the assumption on admissibility of $S_k$ (Sec.~\ref{sec:ac:recursive_admissibility}), and that the basin of attraction of the resulting system is no smaller than the original complex MPC system and possibly larger since the horizon length could be expanded with the additional computational capabilities (Sec.~\ref{sec:ac:basin}).

Before delving into analysis, we must clarify the notion of stability shown and the assumptions its proof requires. Recall that stability for nonlinear dynamical systems is a local property of an equilibrium point rather than a global property of the system. For simplicity, we will follow the proof presented in \cite{allgower2012nonlinear} and show that the origin of the complex system is asymptotically stable by leveraging a cost function that penalizes deviations from the origin. This ``tracking MPC'' formulation \cite{rawlings2017model} requires an additional assumption on how the heuristic lift $\psi^\dagger$ and the cost function $V_I$ interact, specifically that projecting to the simple manifold does not increase the cost:

\begin{assumption} \label{as:ac:psi_dagger_reduces_cost}
$V_{I}^c(z^c) - V_{I}^c(\psi^\dagger \circ \psi(z^c)) \geq 0.$
\end{assumption}

This can be easily obtained by assigning $\psi^\dagger$ to map to the origin for components in the null space of $\psi$, which by definition would have zero cost. 

In the case of legged systems -- such as those that will be introduced in Sec.~\ref{sec:ac:application_to_legged_systems} -- or other systems that seek to track reference trajectories rather than fixed points, the heuristic lift $\psi^\dagger$ could map to the reference trajectory rather than the origin, resulting in a time-varying system. Also note that other formulations exist which yield provable stability results for non-tracking cost functions -- such as ``economic MPC'' \cite{ferramosca2010economic} -- and thus could possibly yield provable stability under more relaxed assumptions on the form of $\psi^\dagger$. Theoretical analysis of these extensions will be left to future work to focus on the most straightforward application of ACMPC.

\subsection{Optimal Control Problem Constraint Satisfaction}
\label{sec:ac:constraint_satisfaction}

We begin by showing that admissibility of $S_k$ results in a lifted trajectory which matches the evolution of the closed-loop dynamics of the actual complex system, and therefore the constraints of the original OCP in (\ref{eq:ac:ocp}) are satisfied by solutions of the ACOCP in (\ref{eq:ac:adaptive_ocp}). Let the evolution of the closed-loop dynamics starting at state $x^c_k$ under a given control trajectory $\mathbf{u}^c$ for duration $i$ be expressed by the function $\phi_{f} : \mathbb{Z} \times \mathcal{X}^c \times \mathbf{U} \rightarrow \mathcal{X}^c$ defined as $\phi_{f}(i,x^c_k,\mathbf{u}^c) \coloneqq f^c( \dots f^c(f^c(x^c_k, u^c_0), u^c_1), \dots, u^c_i)$.

\begin{proposition} \label{th:lifted_state} Suppose Assumption~\ref{as:ac:S_admissible} is satisfied. The predicted state at time $i$ is equal to the solution to the complex dynamical system under the lifted predicted controls, i.e.\ $x^l_{i} = \phi_{f}(i,x^c_0,\mathbf{u}^l)$. 
\end{proposition}
\begin{proof}
We prove this by induction. For the base case $i = 0$, since $0 \notin S_k$, $x^l_{0} = x^c_{0} = \phi_{f}(0,x^c_0,\mathbf{u}^l)$. For the induction step we need to show that $x^l_{i} = \phi_{f}(i,x^c_0,\mathbf{u}^l)$ implies $x^l_{i+1} = \phi_{f}(i+1,x^c_0,\mathbf{u}^l)$. We obtain $\phi_{f}(i+1,x^c_0,\mathbf{u}^l)$ by applying the closed-loop complex dynamics to $\phi_{f}(i,x^c_0,\mathbf{u}^l)$ with the control determined by $\mathbf{u}^l$ to both sides,
\begin{align}
    x^l_{i} &= \phi_{f}(i,x^c_0,\mathbf{u}^l) \\
    f^c(x^l_{i},\mathbf{u}^l_{i}) &= f^c(\phi_{f}(i,x^c_0,\mathbf{u}^l),\mathbf{u}^l_{i}) \\ 
    f^c(z^l_{i}) &= \phi_{f}(i+1,x^c_0,\mathbf{u}^l)
\end{align}
Thus we need to show that $x^l_{i+1} = f^c(z^l_{i})$ to verify that the induction step holds. We proceed by cases based on inclusion in $S_k$.
\begin{case}
$i \notin S_k, i+1 \notin S_k$.
This case corresponds to a portion of the trajectory entirely in the complex space. By the definition of the adaptive system dynamics in (\ref{eq:ac:adaptive_system_dynamics}),
\begin{align}
   x^l_{i+1} &= x^a_{i+1} 
   = f^c(z^a_{i}) 
   = f^c(z^l_{i})
\end{align}
\end{case}
\begin{case}
$i \notin S_k, i+1 \in S_k$.
This case corresponds to a portion of the trajectory which decreases in complexity. By the definition of the adaptive system dynamics in (\ref{eq:ac:adaptive_system_dynamics}) and the construction of $S_k$ in (\ref{eq:ac:admissibility_conditions}),
\begin{align}
   x^l_{i+1} &= \psi^\dagger_x ( x^a_{i+1}) 
   = \psi^\dagger_x \circ \psi \circ f^c (z^a_{i})  
   = f^c(z^a_{i}) 
   = f^c(z^l_{i})
\end{align}
\end{case}
\begin{case}
$i \in S_k, i+1 \notin S_k$.
This case corresponds to a portion of the trajectory which increases in complexity. By the definition of the adaptive system dynamics in (\ref{eq:ac:adaptive_system_dynamics}) and the construction of $S_k$ in (\ref{eq:ac:admissibility_conditions}),
\begin{align}
   x^l_{i+1} &= x^a_{i+1} 
   = f^c \circ \psi^\dagger (z^a_{i})  
   = f^c(z^l_{i})
\end{align}
\end{case}
\begin{case}
$i \in S_k, i+1 \in S_k$.
This case corresponds to a portion of the trajectory entirely in the simple space. By the definition of the adaptive system dynamics in (\ref{eq:ac:adaptive_system_dynamics}) and the construction of $S_k$ in (\ref{eq:ac:admissibility_conditions}),
\begin{align}
   x^l_{i+1} &= \psi^\dagger_x ( x^a_{k+1}) 
   = \psi^\dagger_x \circ f^s (z^a_{i}) 
   = \psi^\dagger_x \circ f^s \circ \psi (z^l_{i})  
   = f^c(z^l_{i})
\end{align}
\end{case}
\setcounter{case}{0}
Thus the induction step holds, completing the proof.
\end{proof}
We must also show that state and control trajectories which satisfy the adaptive state and control constraints $\mathcal{Z}^a_{f}$ also satisfy the complex equivalent $\mathcal{Z}^c_{f}$ by nature of the admissibility of set $S_k$.

\begin{proposition} \label{th:lifting_feasibility}
Suppose Assumption~\ref{as:ac:S_admissible} is satisfied. If $z^a_{i} \in \mathcal{Z}^a_{f}$, then $z^l_{i} \in \mathcal{Z}^c_{f}$.
\end{proposition}
\begin{proof}
Proceed by cases based on inclusion in $S_k$.
\begin{case}
$i \notin S_k$.
   In this case, $z^l_{i} = z^a_{i}$. By the definition of $\mathcal{Z}^a_{f}$ in (\ref{eq:ac:adaptive_feasible_set}), $z^l_{i} \in \mathcal{Z}^c_{f}$.
\end{case}
\begin{case}
$i \in S$.
In this case, $z^l_{i} = \psi^\dagger(z^a_{i})$. Since $S_k \in \mathcal{S}_a$, $z^l_{i} \in \mathcal{Z}^c_{f}$ by \eqref{eq:ac:admissibility_feasible}.
\end{case}
\setcounter{case}{0}
\end{proof}
Next we show that satisfying the initial and terminal state constraints in the adaptive system implies satisfaction of the same constraints in the complex space, which is trivial since the initial and terminal states will always be expressed in the complex space by the construction of $S_k$.
\begin{proposition} \label{th:lifting_terminal_state}
Suppose Assumption~\ref{as:ac:S_admissible} is satisfied. If $x^a_0 = x^c_k$ then $x^l_0 = x^c_k$, and if $x^a_{N^a} \in \mathcal{X}^c_{T}$, then $x^l_{N^a} \in \mathcal{X}^c_{T}$.
\end{proposition}
\begin{proof}
Since $i = 0 \notin S_k$, $x^l_0 = x^a_0 = x^c_k$.  Since $i = N^a \notin S_k, x^l_{N^a} = x^a_{N^a} \in \mathcal{X}^c_{T}$.
\end{proof}

\subsection{Adaptive Complexity Feasibility and Stability}
\label{sec:ac:feasibility}

With these propositions in place we can now prove that the OCP defined in (\ref{eq:ac:adaptive_ocp}) is recursively feasible for states in $\mathcal{X}_{N^a}$ and that this set is invariant in the complex system. This is done by following the form of the proofs in \cite{allgower2012nonlinear}, which constructs a feasible solution (in the absence of modeling errors) to the OCP at the successor state by combining the current solution with the terminal set feedback policy $u_{T}$. This approach is illustrated in Fig.~\ref{fig:ac:proof_illustration}. Define this control sequence as $\Tilde{\mathbf{u}}^a$ and its lifted counterpart $\Tilde{\mathbf{u}}^l$,
\begin{align}
    \Tilde{\mathbf{u}}^a(x^c_{k}) &= \left[ u^{*a}_{1}, \dots, u^{*a}_{N^a-1}, u_{T}(x_{T}) \right] \label{eq:ac:successor_control_seq} \\ 
    \Tilde{\mathbf{u}}^l(x^c_{k}) &= \psi_u^\dagger(\Tilde{\mathbf{u}}^a(x^c_{k}))
\end{align}
where $x_{T} = \phi_{f}(N^a,x^c_k,\mathbf{u}^{*l})$ is the terminal state at time $N^a$ resulting from initial state $x^c_{k}$ and lifted control trajectory $\mathbf{u}^{*l}(x^c_{k})$ (under Proposition~\ref{th:lifted_state}). We now show that this control satisfies the requirements for recursive feasibility.

\begin{figure}[t]
  \centering
  \includegraphics[width=1.0\linewidth]{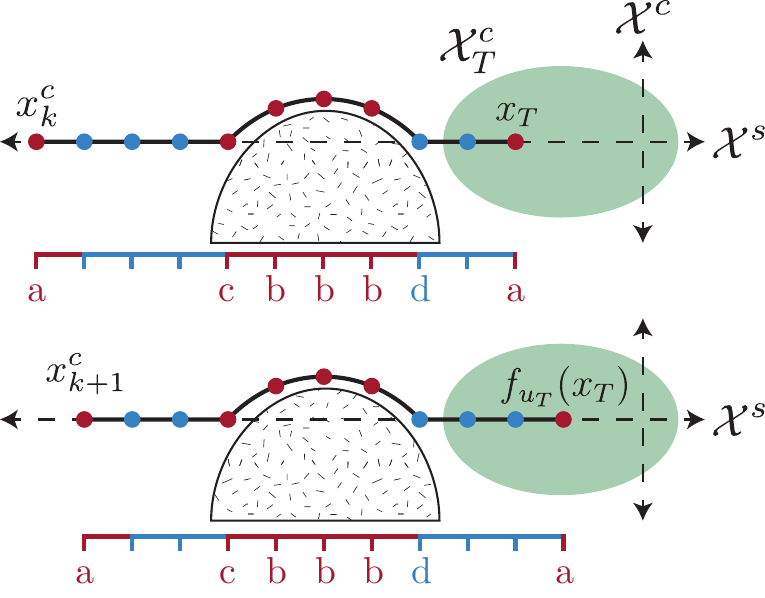}
  \caption[Illustration of adaptive complexity recursive stability]{Adaptive complexity MPC retains recursive feasibility and admissibility by updating the simplicity set and solution at time $k+1$ with the corresponding terms from time $k$ along with the new state and control determined from the terminal policy $u_{T}$ which is applied with the terminal set $\mathcal{X}^c_{T}$. Letters at the bottom indicate the condition of \eqref{eq:ac:admissibility_conditions} that element satisfies or violates. The complex manifold $\mathcal{X}^c$ is the 2D space while $\mathcal{X}^s$ is the embedded 1D submanifold.}
  \label{fig:ac:proof_illustration}
\end{figure}

\begin{proposition} \label{th:ac:recursive_feasibility}
Suppose Assumptions~\ref{as:ac:system_conditions} -- \ref{as:ac:S_admissible} are satisfied. Let $x^c_{k} \in \mathcal{X}_{N^a}$ and let $x^c_{k+1} \coloneqq f^c_{\textnormal{acmpc}}(x^c_{k})$ denote the successor state (under adaptive complexity model predictive control) to $x^c_{k}$. Then $\Tilde{\mathbf{u}}^a(x^c_{k})$ defined in (\ref{eq:ac:successor_control_seq}) is feasible for $\textnormal{ACOCP}_{N^a}(f^c_{\textnormal{acmpc}}(x^c_{k}))$ and $\mathcal{X}_{N^a}$ is positively invariant (for the system $x^c_{k+1} = f^c_{\textnormal{acmpc}}(x^c_{k})$).
\end{proposition}
\begin{proof}
This proof follows standard methods for demonstrating recursive feasibility \cite{allgower2012nonlinear}, see Appendix \ref{sec:ac:appendix_stability_proofs} for the complete proof.
\end{proof}

We must also show that the cost function decreases along any solution of $x^c_{k+1} = f^c_{\textnormal{acmpc}}(x^c_{k})$ given these previous assumptions as this is necessary for the stability proof.

\begin{proposition} \label{th:ac:decreasing_cost}
Suppose Assumptions~\ref{as:ac:system_conditions} -- \ref{as:ac:psi_dagger_reduces_cost} are satisfied. Then
\begin{align} \label{eq:ac:decreasing_cost}
    V^{*a}_N(f^c_{\textnormal{acmpc}}(x^c_{k})) - V^{*a}_N(x^c_{k}) \leq -V_{I}^c(x^c_{k},u^c_{\textnormal{acmpc}}(x^c_{k}))
\end{align}
\end{proposition}
\begin{proof}
See Appendix \ref{sec:ac:appendix_stability_proofs}.
\end{proof}

We can now prove asymptotic stability of the origin of the closed loop system using standard Lyapunov-based methods. This supports Hypothesis~\ref{hyp:ac:conditions} which states that adaptive complexity MPC yields provable stability properties reliant on template and anchor conditions.

\begin{theorem} \label{th:ac:adaptive_complexity_stability}
Suppose Assumptions~\ref{as:ac:system_conditions} -- \ref{as:ac:S_admissible} are satisfied. Then there exists functions $\alpha_1, \alpha_2, \alpha_3 \in \mathcal{K}_\infty$ which upper and lower bound the cost, i.e.,
\begin{subequations}
\label{eq:ac:stability}
\begin{align}
    \alpha_1(|x^c_{k}|) \geq V^{*a}_{N^a}(x^c_{k}) &\geq \alpha_2(|x^c_{k}|) \label{eq:ac:stability_inequality} \\
    V^{*a}_{N^a}(x^c_{k+1}) - V^{*a}_{N^a}(x^c_{k}) &\leq -\alpha_3(|x^c_{k}|) \label{eq:ac:stability_decreasing_cost}
\end{align}
\end{subequations}
and thus the origin of the system,
\begin{align}
    x^c_{k+1} = f^c_{\textnormal{acmpc}}(x^c_{k})
\end{align}
is asymptotically stable with a region of attraction $\mathcal{X}_{N^a}$.
\end{theorem}
\begin{proof}
See Appendix \ref{sec:ac:appendix_stability_proofs}.
\end{proof}

\subsection{Recursive Admissibility of $S_k$}
\label{sec:ac:recursive_admissibility}

Theorem~\ref{th:ac:adaptive_complexity_stability} shows that adaptive complexity MPC is stable under Assumption~\ref{as:ac:S_admissible} which states that $S_k$ is admissible. Next we prove that this holds for all time under Algorithm~\ref{alg:ac:acmpc}, again assuming no modeling errors. Robustness considerations remain an intriguing area for future investigation.
\begin{lemma} \label{th:ac:recursive_admissibility}
Let $x^c_0 \in \mathcal{X}_{N^a}$. Then $S_k \in \mathcal{S}_a ~\forall k \geq 0$ under Algorithm~\ref{alg:ac:acmpc}.
\end{lemma}
\begin{proof}
We proceed by induction. The base case $k = 0$ is met by the assumption that $S_0 \in \mathcal{S}_a$. For the induction step we must show that $S_k \in \mathcal{S}_a$ implies $S_{k+1} \in \mathcal{S}_a$.
By Proposition~\ref{th:ac:recursive_feasibility}, $z^{*l}_{k+1}$ consists of each of the last $N^a - 1$ elements of $z^{*l}_{k}$, plus the new terminal state-control pair $(x_{T}, u_{T}(x_{T}))$. Since under Algorithm~\ref{alg:ac:acmpc} elements of $S^a_{k+1}$ are the time-shifted elements of $S^a_{k}$ which satisfy the admissibility conditions (\ref{eq:ac:admissibility_conditions}), these conditions are satisfied for all $i \in S^a_{k+1}$. Since by the definition of $S_{k+1} = S^a_{k+1} \cap S^f$ where $i = N^a \notin S^f$, the new terminal state-action pair is always in the complex space, and hence (\ref{eq:ac:admissibility_conditions}) are satisfied for all $i \in S_{k+1}$ and thus $S_{k+1} \in \mathcal{S}_a$.
\end{proof}

Note that the terminal state-control pair may not meet the reduction conditions in (\ref{eq:ac:admissibility_conditions}), meaning that it must remain in the complex space. This is handled by assuming the last finite element in the horizon is complex, and checking (\ref{eq:ac:admissibility_conditions}) after solving the OCP to determine if the new index can be allowed into $S_{k}$ or must remain complex.

\subsection{Basin of Attraction Comparison}
\label{sec:ac:basin}

We have shown that both the original MPC formulation for the complex system in (\ref{eq:ac:closed_loop_system}) and the adaptive complexity MPC system in (\ref{eq:ac:adaptive_system}) are asymptotically stable about the origin with basins of attraction $\mathcal{X}_{N^c}$ and $\mathcal{X}_{N^a}$, respectively. We now show that the size of their basins of attraction is dependent on the horizon lengths $N^c$ and $N^a$.

\begin{lemma} \label{th:ac:basin_of_attraction}
If $N^c \leq N^a$ then $\mathcal{X}_{N^c} \subseteq \mathcal{X}_{N^a}$, and if $N^c < N^a$ then  $\mathcal{X}_{N^c} \subset \mathcal{X}_{N^a}$.
\end{lemma}
\begin{proof}
Let $x^c_0 \in \mathcal{X}_{N^c}$, thus $\text{OCP}_{N^c}(x^c_0)$ has a solution $\mathbf{u}^*$ satisfying the constraints in (\ref{eq:ac:ocp}). Let $S_0$ be the initial admissibility set of this solution determined by evaluating the conditions in (\ref{eq:ac:admissibility_conditions}). Construct an adaptive control trajectory $\mathbf{u}^a$,
\begin{align}
    u^a_i &= \begin{cases}
        u^*_i & i \notin S_0  \\ 
        \psi_u(u^*_i) & i \in S_0
    \end{cases}
\end{align}
for all $i = 0, 1, \dots, N^c-1$, which implies $\mathbf{u}^l = \mathbf{u}^*$. By Propositions \ref{th:lifted_state} -- \ref{th:lifting_terminal_state} the constraints on the OCPs (\ref{eq:ac:ocp}) and (\ref{eq:ac:adaptive_ocp}) are equivalent, and thus $\mathbf{u}^a$ is a valid solution of (\ref{eq:ac:adaptive_ocp}) which implies that $x^c_0 \in \mathcal{X}_{N^a}$ and thus $\mathcal{X}_{N^c} \subseteq \mathcal{X}_{N^a}$.

To show that $N^c < N^a \rightarrow \mathcal{X}_{N^c} \subset \mathcal{X}_{N^a}$, consider a point $x^c_{0} \notin \mathcal{X}_{N^c}$ and some feasible control $u^c_0 \in \mathcal{U}_{f}$ such that $z^c_0 \coloneqq (x^c_0, u^c_0) \in \mathcal{Z}^c_{f}$ and $x^c_1 = f^c(z^c_0) \in \mathcal{X}_{N^c}$. Let $u^*_1$ be the control sequence yielded by solving $\text{OCP}_{N^c}(x^c_1)$. Let the control sequence $\mathbf{u}^a_0 = [u^c_0, \mathbf{u}^a_1]$, and let $N^a = N^c + 1$. Since the state $x^c_1 \in \mathcal{X}_{N^a}$, by Propositions \ref{th:lifted_state} -- \ref{th:lifting_terminal_state} the state, control and terminal constraints of (\ref{eq:ac:adaptive_ocp}) are satisfied for $z^a_i~\forall i = 1, 2, \dots, N^c$, and since $z^c_0 \in \mathcal{Z}^c_{f}$, the state and control constraints are satisfied for $i = 0$. Thus all the constraints of (\ref{eq:ac:adaptive_ocp}) are satisfied, and therefore $x^c_0 \in \mathcal{X}_{N^a}$. The property $\mathcal{X}_{N^c} \subset \mathcal{X}_{N^a}$ follows.
\end{proof}

The property that longer horizon lengths yield larger basins of attraction is a known result of MPC theory -- as horizon lengths go to infinity, MPC converges to infinite-horizon optimal control \cite{allgower2012nonlinear}. Horizon lengths are generally limited by computational effort, so if simplifying the problem allows for longer horizons under equal computational capabilities without violating system constraints, the resulting controller has a larger basin of attraction. We note this theoretical result, although in the following experiments we leave horizon length fixed and allow solve time variations rather than the reverse as this yields a more consistent representation of computational constraints. Also note that since membership in the basin of attraction implies that a solution to the OCP exists, Lemma~\ref{th:ac:basin_of_attraction} implies that adaptive complexity maintains the completeness properties of the original MPC formulation in the complex system (assuming a suitable algorithm which can solve the OCP in \eqref{eq:ac:adaptive_ocp}).

Additionally, it should be noted that although feasible solutions of the ACOCP are also feasible solutions of the original complex OCP, they may not be optimal with respect to the original OCP cost function. Reducing the system at a particular time is in essence constraining it to the simple manifold at that time. The conditions stated above require that doing so be feasible and yield a decreasing cost, yet it is possible that doing so could yield a trajectory of higher cost across the whole trajectory than if this constraint were lifted. As such this approach may result in a sacrifice in cost optimality in favor of a simpler problem.

\section{Application to Legged Systems} \label{sec:ac:application_to_legged_systems}
To demonstrate the validity of the proposed adaptive complexity MPC in controlling dynamical systems and to provide examples for the quantities defined above, we apply the approach for a legged robot system. This section defines the complex and simple models which are used for implementing the algorithm, as well as the mappings between them. The complex model includes states and constraints on both the body and feet of the robot (and thus also joint information through kinematics calculations) whereas the simple model only considers body motion.

\subsection{Definition of Complex Legged System}
The complex system represents a model of the robot that includes body and foot states as shown in Fig.~\ref{fig:ac:legged_systems}. More precisely, we leverage a single rigid body dynamics model and a full kinematics model, similar to \cite{dai2014whole}. We choose to include foot states in the dynamical system rather than joint states for two reasons. First, we note that joint information is available in the OCP via slack variables constrained by forward kinematics, and that operating in end effector space renders certain functions -- in particular approximated limb dynamics and terrain clearance -- linear with respect to state variables and therefore more conducive to online replanning. Second, this parameterization enables a quasi-massless leg assumption, where limb dynamics can still be approximated (in this case with a point-mass model) without affecting body dynamics, which is desirable to satisfy \eqref{eq:ac:admissibility_anchor} as much as possible. Note that joint dynamics could be accounted for by using a centroidal dynamics model for the simple system and whole-body dynamics as the complex system. This approach would allow the simple model to ignore joint dynamics when feasible while still capturing the effects of limb inertia and kinematics on body dynamics, and only compute joint dynamics when needed to satisfy actuation constraints. However, since the massless leg assumption has been shown to be useful for systems like the quadrupedal one studied here \cite{wensing2022optimization, paper:johnson_hs_2016}, and we are primarily interested in behaviors where limb kinematics are more relevant than their dynamics, we leave the extension to other systems as future work.

We define the states of the complex system $x^c \in \mathbb{R}^{12 + 6n}$ with $n$ denoting the number of legs,
\begin{align} \label{eq:ac:complex_system_state}
    x^c = \left[\begin{array}{cccccc}
         q_{\text{lin}}^T &
         q_{\text{ang}}^T &
         q_{\text{foot}}^T &
         \dot{q}_{\text{lin}}^T &
         \omega^T &
         \dot{q}_{\text{foot}}^T
    \end{array} \right]^T
\end{align}
where $q_{\text{lin}} \in \mathbb{R}^3$ defines the body position in the world frame, $q_{\text{ang}} \in \mathbb{R}^3$ defines the body orientation through a local parameterization such as Euler angles, $q_{\text{foot}} \in \mathbb{R}^{3n}$ defines the foot positions in the world frame, and $\omega \in \mathbb{R}^3$ is the angular velocity in the body frame (with $\widehat{\omega}$ its skew-symmetric equivalent). Let the configuration be defined as $q^c \coloneqq [q^T_{\text{lin}}, q^T_{\text{ang}}, q^T_{\text{foot}}]^T$.

\begin{figure}[t]
    \centering
    \subfloat[\centering Complex system]{{\includegraphics[width=5cm]{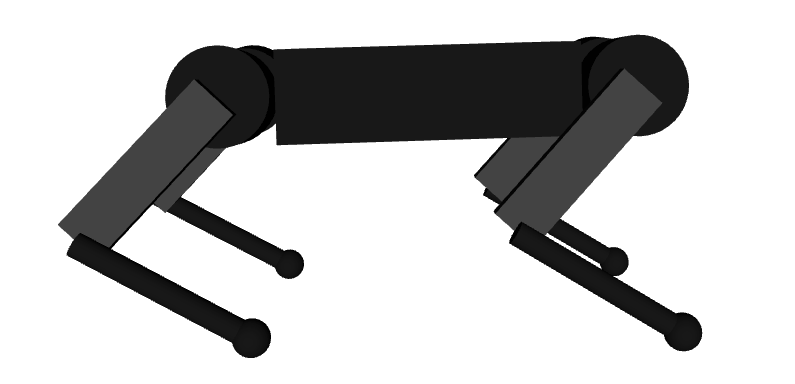} }}%
    \qquad
    \subfloat[\centering Simple system]{{\includegraphics[width=5cm]{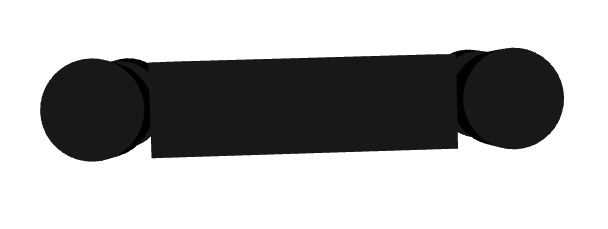} }}%
    \caption[Complex and simple legged systems]{The complex model includes body and foot states which together define joint states, while the simple model only consists of body states.}%
    \label{fig:ac:legged_systems}%
\end{figure}

Let the control inputs $u^c \in \mathbb{R}^{6n}$ be defined as,
\begin{align} \label{eq:ac:complex_system_control}
    u^c = \left[\begin{array}{c}
         u_{\text{body}} \\
         u_{\text{foot}} 
    \end{array} \right]
\end{align}
where $u_{\text{body}} \in \mathbb{R}^{3n}$ are the desired ground reaction forces at each foot in the world frame coordinates, and $u_{\text{foot}} \in \mathbb{R}^{3n}$ denote the forces which accelerate the feet during swing, but do not act on the robot body.
Note that these forces on the physical system correspond to the same actuators as any given leg is either in stance or swing -- this separation is primarily to distinguish between control authority available in the simple and complex systems.

Let $u_{\text{body},j} \in \mathbb{R}^{3}$ and $q_{\textnormal{foot}, j} \in \mathbb{R}^{3}$ be the ground reaction forces and foot positions respectively for stance leg $j$. Using a simple point-mass model for the feet and following the standard formulation for single rigid body dynamics models, e.g.\ as in \cite{chignoli2020variational}, define the continuous time dynamics of this system $f^c(x^c,u^c)$,
\begin{align} \label{eq:ac:legged_complex_dynamics}
    f^c(x^c,u^c) &=
\renewcommand{\arraystretch}{1.2}
    \left[ \begin{array}{c}
        \dot{q}_{\text{lin}} \\
        J_{\omega} (q_{\text{ang}}){\omega} \\
        \dot{q}_{\text{foot}} \\
        \frac{1}{m}\sum_{j} u_{\text{body},{j}} - g \\
        W(q_{\text{lin}}, q_{\text{foot}}, \omega, u_{\text{body}}) \\
        u_{\text{foot}}
    \end{array} \right]
\renewcommand{\arraystretch}{1}
\end{align}
where $J_{\omega} (q_{\text{ang}}) \in SO(3)$ is the linear mapping from angular velocity to the derivative of the orientation local parameterization, $m$ is the body mass, $g$ is the gravity vector, and the shorthand function $W(\cdot)$ maps the state and control to the angular acceleration,
\begin{align} \label{eq:ac:ang_acc_map}
    W(q_{\text{lin}}, &q_{\text{foot}}, \omega, u_{\text{body}}) = \nonumber\\ &M^{-1}\Big(R(q_{\text{ang}})^T\sum_j((q_{\text{foot},j} - q_{\text{lin}}) \times u_{\text{body},j} ) - \widehat{\omega} M \omega\Big)
\end{align}
where $M$ is the inertia matrix in the body frame and $R(q_{\text{ang}}) \in SO(3)$ maps vectors from the body frame to the world frame. The discrete time formulation of the dynamics in (\ref{eq:ac:legged_complex_dynamics}) can be obtained with a suitable integration scheme such as forward Euler.

Next, we define the constraints of the complex system $\mathcal{Z}^c$. Kinematic constraints for legged systems are generally functions of joint limits rather than body or foot variables. As a result, we add the joint information to the OCPs as slack variables and use them to define constraints. Let $\theta, \dot{\theta}, \tau \in \mathbb{R}^{n\cdot n_{j}}$ be the joint positions, velocities, and torques, respectively, where $n_j$ is the number of joints per leg. These slack variables are constrained,
\begin{subequations}
\label{eq:ac:slack_variables}
\begin{align}
    FK(q^{c},\theta) &= q_{\text{foot}} \label{eq:ac:slack_variables_fk} \\
    J(q^{c},\theta)[\dot q_{\text{lin}}^T, \omega^T, \dot \theta^T]^T &= \dot q_{\text{foot}} \label{eq:ac:slack_variables_fk_vel} \\
    \tau &= -J_{\theta}(q^{c},\theta)^T u_{\text{body}} \label{eq:ac:slack_variables_torque}
\end{align}
\end{subequations}
where $FK(q^{c},\theta)$ is the forward kinematics function of the system, $J(q^{c},\theta)$ is the leg Jacobian which relates motion of the body and joints to foot motion in the world frame, and $J_{\theta}$ are the columns of this matrix corresponding to joint motion. Note that \eqref{eq:ac:slack_variables_torque} omits joint dynamics under the massless leg assumption.  As discussed at the beginning of this section, these additional variables and constraints allow the optimization to determine joint kinematics which correspond to a given set of body and foot kinematics. This in turn enables satisfaction of kinematic limits without requiring their direct inclusion in the dynamical system model.

With these slack variables and their corresponding constraints in place, we can now state the constraints in the complex system,
\begin{subequations}
\label{eq:ac:legged_complex_constraints}
\begin{align}
    \theta_{\text{min}} \leq \theta& \leq \theta_{\text{max}} \\
    \dot{\theta}_{\text{min}} \leq \dot{\theta}& \leq \dot{\theta}_{\text{max}} \\
    \tau_{\text{min}} \leq \tau& \leq \tau_{\text{max}} \label{eq:ac:torque_limit}\\
    u_{\text{body},\text{min}} \leq u_{\text{body}}& \leq u_{\text{body},\text{max}} \label{eq:ac:grf_limit} \\
    u_{\text{body}}& \in \mathcal{FC} \label{eq:ac:friction}\\
    D u_{\text{body}}& = 0 \label{eq:ac:contact}\\
    \bar{D} \dot{q}_{\text{foot}}& = 0 \label{eq:ac:contact_position}\\
    -\tau_{\text{max}}\left(1 + \frac{\dot{\theta}}{\dot{\theta}_\text{max}}\right) \leq \tau& \leq  \tau_{\text{max}}\left(1 - \frac{\dot{\theta}}{\dot{\theta}_\text{max}}\right) \label{eq:ac:motor_model} \\
    h(q_{\text{foot}})& \geq 0 \label{eq:ac:ground_height}
\end{align}
\end{subequations}
where $(\cdot)_\text{min}$ and $(\cdot)_\text{max}$ represent variable bounds, 
\eqref{eq:ac:friction} enforces that the GRF at each foot lies within the linearized non-adhesive friction cone $\mathcal{FC}$,
\eqref{eq:ac:contact} and \eqref{eq:ac:contact_position} enforce a contact schedule with selection matrix $D$ and its bit-wise complement $\bar{D}$,
 \eqref{eq:ac:motor_model} enforces a linear motor model, and \eqref{eq:ac:ground_height} enforces non-penetration of the terrain via the ground clearance $h(q_{\text{foot}})$. Note that GRF bounds are retained to ensure non-adhesion and discourage abuse of kinematic singlarities. Together the constraints in \eqref{eq:ac:legged_complex_constraints} define the set $\mathcal{Z}^c_{f}$.

Lastly we must define the cost functions for the OCP. While this is still an active area of research for legged systems, a common approach is to extract a reference trajectory either from a higher-level motion planner or by integrating forward a desired body velocity, supplement it with task-space foot trajectories, then penalize any deviations \cite{norby2022quad, grandia2022perceptive, sleiman2021unified, mastalli2020crocoddyl}. Let those trajectories be described by $\bar{\mathbf{x}}^c, \bar{\mathbf{u}}^c$. Under this approach, the intermediate and terminal costs for the OCP of the complex system can be defined as,
\begin{subequations} \label{eq:ac:legged_complex_cost}
\begin{align}
    e^{c}_{x,k} &= x^c_{k} - \bar{x}^c_{k} \\
    e^{c}_{u,k} &= u^c_{k} - \bar{u}^c_{k} \\
    V_{I}^c(x^c_{k}, u^c_{k}, k) &= {e^{c}_{x,k}}^T Q_{I} e^{c}_{x,k} + {e^{c}_{u,k}}^T R e^{c}_{u,k} \\
    V^c_{T}(x^c_{k}, k) &= {e^{c}_{x,k}}^T Q_{T} e^{c}_{x,k}
\end{align}
\end{subequations}
where $e^{c}_{x,k}$ and $e^{c}_{u,k}$ are the respectively state and input errors from the time-varying reference, and where $Q_{I}, Q_{T}$, and $R$ are positive definite matrices. Note that while these matrices must be positive definite to yield provable stability, often assigning very low costs to particular components will yield behavior more consistent with online motion planners than tracking controllers (leveraging available degrees of freedom to minimize the overall cost).

\subsection{Definition of Simple Legged System}
The simple system represents the reduced-order model of the robot that uses only the body states and ignores the foot and joint states, as shown in Fig.~\ref{fig:ac:legged_systems}. This is a commonly used model reduction technique in the legged locomotion literature. The state of the simple system is defined as $x^s \in \mathbb{R}^{12}$ consisting of only the body position, orientation, and velocities, 
\begin{align} \label{eq:ac:simple_system_state}
    x^s = \left[\begin{array}{cccc}
         q_{\text{lin}}^T &
         q_{\text{ang}}^T &
         \dot{q}_{\text{lin}}^T &
         \omega^T
    \end{array} \right]^T
\end{align}
The control inputs $u \in \mathbb{R}^{3n}$ of the simple system consists of only the GRFs from the complex system, such that $u^s = u_{\text{body}}$. The dynamics in the simple system are the components of the complex dynamics corresponding to the simple system states defined as $f^s(x^s, u^s, k)$,
\begin{align} \label{eq:ac:legged_simple_dynamics}
    f^s(x^s,u^s, k) &=
\renewcommand{\arraystretch}{1.2}
    \left[ \begin{array}{c}
        \dot{q}_{\text{lin}} \\
        J_{\omega} (q_{\text{ang}}){\omega} \\
        \frac{1}{m}\sum_{j} u_{\text{body},{j}} - g \\
        W(q_{\text{lin}}, \bar{q}_{\text{foot},k}, \omega, u_{\text{body}})
    \end{array} \right]
\renewcommand{\arraystretch}{1}
\end{align}
Note that in the simple system, $\bar{q}_{\text{foot},k}$ is a parameter (not a decision variable), which is necessary for the angular acceleration update, and renders the dynamics time-varying. The constraints in the simple space are only constraint bounds on the input $u_{\text{body}}$, identical to equations \eqref{eq:ac:grf_limit}--\eqref{eq:ac:contact}.

\subsection{Relations Between Complex and Simple Legged Systems}
With these systems defined, we can now relate the two with the reductions $\psi_x, \psi_u$ and define our heuristic lifts $\psi^\dagger_x, \psi^\dagger_u$. The reductions select components to retain within the simple system,
\begin{subequations} \label{eq:ac:legged_reductions}
\begin{align}
    \psi_x(x^c) &= 
    \left[\begin{array}{cccccc}
        I_{3} & 0 & 0 & 0 & 0 & 0 \\
        0 & I_{3} & 0 & 0 & 0 & 0 \\
        0 & 0 & 0 & I_{3} & 0 & 0 \\
        0 & 0 & 0 & 0 & I_{3} & 0
    \end{array} \right] x^c = x^s \\
    \psi_u &= \left[ \begin{array}{cc}
        I_{3n} & 0  
    \end{array} \right] u^c = u^s
\end{align}
\end{subequations}
where $I_m$ is an $m \times m$ identity matrix and $0$ is a matrix of zeros of corresponding size.

As discussed in Section~\ref{sec:ac:theoretical_analysis}, a sensible choice for the heuristic lift $\psi^\dagger$ is a mapping which preserves the states in the simple system and fills in null space components with values from the reference $\bar{\mathbf{z}}^c$. If the reference is time-varying (as is the case for legged systems), this results in time-varying lifts,
\begin{subequations} \label{eq:ac:legged_lifts}
\begin{align}
    \psi^\dagger_x(x^s, k) = &
    \left[\begin{array}{cccc}
        I_3 & 0 & 0 & 0 \\
        0 & I_3 & 0 & 0 \\
        0 & 0 & 0 & 0 \\
        0 & 0 & I_3 & 0 \\
        0 & 0 & 0 & I_3 \\
        0 & 0 & 0 & 0 \\
    \end{array} \right] x^s + \nonumber\\ &
    \left[\begin{array}{cccccc}
        0 & 0 & 0 & 0 & 0 & 0 \\
        0 & 0 & 0 & 0 & 0 & 0 \\
        0 & 0 & I_{3n} & 0 & 0 & 0 \\
        0 & 0 & 0 & 0 & 0 & 0 \\
        0 & 0 & 0 & 0 & 0 & 0 \\
        0 & 0 & 0 & 0 & 0 & I_{3n} \\
    \end{array} \right] \bar{x}^c_{k} \\
    \psi^\dagger_u(u^s, k) &= \left[\begin{array}{c}
        I_{3n} \\
        0
    \end{array} \right] u^s +
    \left[\begin{array}{cc}
        0 & 0\\
        0 & I_{3n}
    \end{array} \right] \bar{u}^c_{k}
\end{align}
\end{subequations}
Note that using dynamically consistent reference trajectories when defining this lift increases the likelihood that simplicity set conditions will be met, since they require that states on the simple manifold will remain there. Also since we will often want to compute equivalent joint data to accompany these null space variables, these functions can be post-processed with inverse kinematics and inverse dynamics to compute the values of $\theta$, $\dot{\theta}$, and $\tau$ that correspond to a particular value of $z^c_{k}$.

We can now state the conditions for admissibility of this system, which require that the variables in the null space of $\psi$ lie on the reference and are feasible.
\begin{lemma} \label{th:ac:legged_conditions}
For a given state-control pair $z^l_i$ which lie on a trajectory of the system defined in \eqref{eq:ac:legged_complex_dynamics}, a reduction at $i \in [1,\dots,N-1]$ is admissible (Definition~\ref{def:ac:admissibility}) if $q_{\textnormal{foot}} = \bar{q}_{\textnormal{foot}}$, $\dot{q}_{\textnormal{foot}} = \dot{\bar{q}}_{\textnormal{foot}}$, $u_{\textnormal{foot}} = \bar{u}_{\textnormal{foot}}$, and $z^l_i$ satisfies the constraints in \eqref{eq:ac:legged_complex_constraints}.
\end{lemma}
\begin{proof}
See Appendix~\ref{sec:ac:legged_conditions_proof}.
\end{proof}

\subsection{ACMPC Algorithm Application}

We can now restate the major components of the ACMPC algorithm in the context of this example, specifically the ACOCP and the simplicity set update.

\subsubsection{Legged System ACOCP}

The components needed to solve the adaptive complexity optimal control problem defined in \eqref{eq:ac:adaptive_ocp} are the dynamics, constraints, costs, and terminal conditions.

The adaptive dynamics $f^a$ match the form of those specified in \eqref{eq:ac:adaptive_system_dynamics}, but leverage the equations defined above, specifically the complex dynamics $f^c$ in \eqref{eq:ac:legged_complex_dynamics}, the simple dynamics $f^s$ in \eqref{eq:ac:legged_simple_dynamics}, the reduction $\psi$ in \eqref{eq:ac:legged_reductions}, and the heuristic lift $\psi^\dagger$ in \eqref{eq:ac:legged_lifts}.

These equations together form the optimization constraint \eqref{eq:ac:adaptive_ocp:dynamics}, which is added in a piece-wise manner to each finite element in the ACOCP. To maintain admissibility properties, we must also enforce that complex-to-simple system transitions always result in successor states that reside on the simple manifold. This can be done by lifting the transition,
\begin{align} \label{eq:ac:legged_adaptive_system_dynamics_reduction_handling}
\psi^\dagger_{x} (x^{a}_{i+1}) &= f^c(z^a_{i}) & i \notin S_k,\, i+1 \in S_k
\end{align}
Practically speaking, this encodes the null space of $\psi$ as optimization parameters (i.e.~$\bar{q}_{\textnormal{foot}}$) rather than decision variables (i.e.~$q_{\textnormal{foot}}$) when evaluating the dynamics constraint defined by \eqref{eq:ac:legged_complex_dynamics}, which effectively requires the control input to drive the state exactly onto the simple manifold.

Constraints are defined based on the simplicity set as described in \eqref{eq:ac:adaptive_feasible_set}, where $\mathcal{Z}^c_f$ is the set of all state-control pairs that satisfy \eqref{eq:ac:slack_variables} and \eqref{eq:ac:legged_complex_constraints}, and $\mathcal{Z}^s_f$ is the set of pairs that satisfy \eqref{eq:ac:grf_limit}-\eqref{eq:ac:contact}. These constraints are then applied piece-wise as in \eqref{eq:ac:adaptive_ocp:constraints}.

Note that this means complicated joint constraint functions are only evaluated for finite elements in the complex system.

Costs are likewise defined as described in \eqref{eq:ac:adaptive_cost} and applied piece-wise as in \eqref{eq:ac:adaptive_ocp:cost}, with $V_{I}^a$ computed with the cost function $V_{I}^c$ defined in \eqref{eq:ac:legged_complex_cost} and the heuristic lift $\psi^\dagger$ in \eqref{eq:ac:legged_lifts}.

The initial condition constraint in \eqref{eq:ac:adaptive_ocp:initial_condition} is applied to the first finite element which is guaranteed to be in the complex space. There are many options for the terminal set constraint in \eqref{eq:ac:adaptive_ocp:terminal_constraint}. In the absence of a terminal control law $u_T$ -- which has no known form for legged systems -- common approaches include constraining the final state to lie on the reference, or increasing the terminal cost $V^c_T$ to strongly penalize solutions which do not drive the system to the reference.

These equations fully specify all the components in \eqref{eq:ac:adaptive_ocp}, which can be evaluated in a piece-wise manner by an off-the-shelf NLP solver. This work transcribes the NLP using direct methods and solves with IPOPT \cite{wachter2006implementation}, but other works have solved similar mixed-complexity problems with indirect, DDP-based methods \cite{li2021model} . Note that the mappings $\psi$ and $\psi^\dagger$ are only needed to evaluate the dynamics constraint when transitioning onto or off of the simple manifold or when evaluating the cost for the simple system.

\subsubsection{Legged System Simplicity Set Update}
Once the ACOCP in \eqref{eq:ac:adaptive_ocp} is solved, the next step in Algorithm~\ref{alg:ac:acmpc} is to evaluate the admissibility conditions. Conditions \eqref{eq:ac:admissibility_anchor} and \eqref{eq:ac:admissibility_enter_manifold} which regulate transitions onto and off of the manifold are automatically enforced via the heuristic values in the dynamics constraint as described in \eqref{eq:ac:legged_adaptive_system_dynamics_reduction_handling}, so the only condition that must be checked is the feasibility condition in \eqref{eq:ac:admissibility_feasible}. This is done by first obtaining $\psi^\dagger \circ \psi (\mathbf{z}^{*l})$, which uses the maps in \eqref{eq:ac:legged_reductions} and \eqref{eq:ac:legged_lifts} to project the foot states and controls to the heuristic (along with joint states and controls through inverse kinematics and dynamics). These state-control pairs can then be checked for feasibility in the complex system using \eqref{eq:ac:legged_complex_constraints}. If any constraints are violated at a time $i$, then both $i$ and $i-1$ are removed from the simplicity set. Removing both knot points ensures that a full finite element of complexity will be added to the problem, which enables states to leave the simple manifold (see Fig.~\ref{fig:ac:S_illustration} for intuition behind this principle). Otherwise, all conditions in \eqref{eq:ac:admissibility_conditions} are met and $i$ can be added to the simplicity set, at which point the control is executed and the algorithm repeats.

\section{Experimental Evaluation} \label{sec:ac:experiments}

This section presents experiments deploying adaptive complexity MPC on a simulated quadrupedal robot to quantify its performance and benchmark against other formulations of MPC. In particular we compare against three other model configurations -- ``Simple'' and ``Complex'' respectively employ only the simple and complex model dynamics and constraints, and ``Mixed'' employs the complex model for the first one-quarter of the horizon and the simple model for remainder, similar to \cite{li2021model}. The ``Adaptive'' configuration follows Algorithm~\ref{alg:ac:acmpc} with $S^f = \{2,..., N^a-2\}$ so that the first and last two knot points -- and thus the first and last finite elements -- are always complex. This slight restriction of the fixed simplicity set is added to improve robustness to modeling errors as discussed in Sec.~\ref{sec:ac:algorithm_overview:S_convergence}. It also ensures that new elements entering the horizon will be complex, which is needed to meet Assumption~\ref{as:ac:S_admissible} in lieu of a hard-to-find terminal controller $u_{T}$ that meets the conditions in Assumption~\ref{as:ac:original_mpc_conditions}.

In each experiment, the robot is provided a reference trajectory which defines the required task over the given environment. We use three environments to evaluate the algorithm performance in the presence of varying constraints. The ``Acceleration'' environment requires the robot to accelerate and decelerate over 7.5 body lengths (3 m) of flat terrain as quickly as possible to measure the ability to stabilize the system during agile motions. The ``Step'' environment consists of a one-half leg length (20 cm) step which requires navigating state constraints such as joint limits to traverse. The ``Gap'' environment consists of a body-length (40 cm) gap which the robot must leap across, testing the controller's ability to handle these kinematic constraints in addition to input constraints such as actuator limits and friction.

Performance is quantified in each experiment by measuring success rate, mean solve time, slow solve rate (\% of solves that take longer than one timestep), and a task-specific metric. Success is defined as reaching the goal within a one-half body length (20 cm) and zero velocity without failing to solve the NLP. Data is recorded from the instant the first NLP is solved until either the success or failure criteria are met, and only included in metrics if successful. In the Acceleration experiment, the task-specific metric is average top speed achievable with a 100 \% success rate over ten trials, which is determined by iteratively decreasing the reference duration until the robot was unable to maintain a 100 \% success rate. The task-specific metric for the Step and Gap experiments is mean control error, which is the norm of the error between the sum of ground reaction forces and the weight of the robot, averaged over the successful trials. For the Step and Gap environments, the robot is initialized to a random position within one-half body length in the transverse plane from a nominal position. In each configuration, solve time corresponds to the time spent in the NLP solver to isolate the effects on the underlying OCP, although the time required to evaluate the conditions in \eqref{eq:ac:admissibility_conditions} is negligible -- our un-optimized implementation takes around 1 ms in the worst case.

In each environment a fixed reference trajectory for the body is provided from the global planner described in \cite{norby2020fast}, and the reference contact sequence and foot trajectories are chosen online before each MPC iteration with a Raibert-like heuristic \cite{raibert1986legged} and a threshold on traversability, as described in \cite{norby2022quad}. Note that these references could have been provided as twist inputs integrated forwards in time from the current state -- we chose to provide fixed references so they would remain invariant across all trials. We fix the prediction horizon at two gait cycles ($N^c = N^a = 24$) with a timestep of $\Delta t = 0.03$ s and measure the relative solve times as discussed in Sec.~\ref{sec:ac:basin}, although future work could implement an adaptive horizon approach to keep solve time fixed. For each solve, we use the state provided by the simulator as the initial state.

Once the reference and state information have been obtained, we construct the NLP with the appropriate complexity structure and solve it with IPOPT \cite{wachter2006implementation}. We configure IPOPT to enable warm start initialization and provide it the primal and dual variables from the prior solve (appropriately shifted) for rapid convergence. In the event that complexity must be unexpectedly inserted in the middle of the horizon (discussed in more detail in Sec.~\ref{sec:ac:discussion}), warm starting is disabled and nominal values are used for the initial guess of the newly inserted null-space variables. Once the problem is solved, the MPC control output is then mapped from ground reaction forces to joint torques via the Jacobian-transpose method, and the resulting swing foot trajectories are tracked with PD control. See \cite{norby2022quad} for more details on the implementation of the low-level controller. All experiments were performed using Gazebo 9 with the ODE physics engine, and all processes were executed on a machine running Ubuntu 18.04 with an Intel Core i7-12700K CPU at 4.9 GHz and with 64 GB of RAM. The motor model described in (\ref{eq:ac:motor_model}) is implemented in the Gazebo simulation to enable a more realistic execution. The Acceleration environment is simulated in real-time to measure the effect of solve time on stability. The simulations were slowed down by a factor of 2x for the Step environment and 5x for the Gap environment (with a maximum solve time of $4\Delta t = 0.12$~s) since resolving the constraints in these tasks are still computationally intensive.

\subsection{Acceleration Environment}
\begin{figure*}[!p]
\centering
\subfloat[1.0\textwidth][The Acceleration environment requires the robot to rapidly move forwards 3 m and come to a rest. Snapshots are equally distributed in time, with increasing opacity corresponding to progression forwards in time.]{\includegraphics[width=1.0\textwidth]{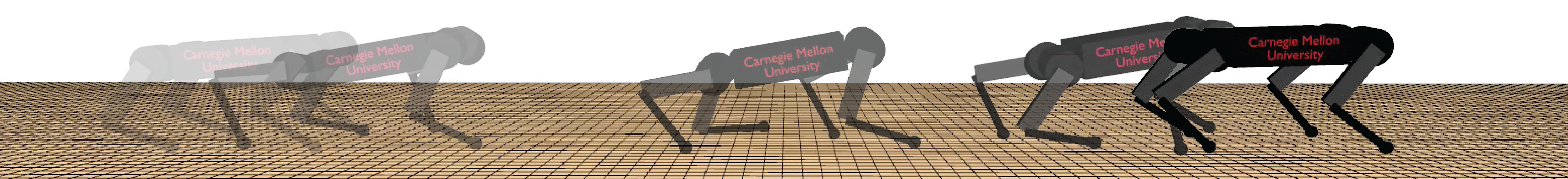}\label{fig:ac:acceleration_sequence}}
\quad
\vskip1em
\subfloat[0.485\textwidth][Acceleration experiment position and velocity trajectories show that the additional computation required by the Complex configuration significantly reduces its performance. Each curve corresponds to one trial at the maximum feasible commanded acceleration.]{\includegraphics[width=0.485\textwidth]{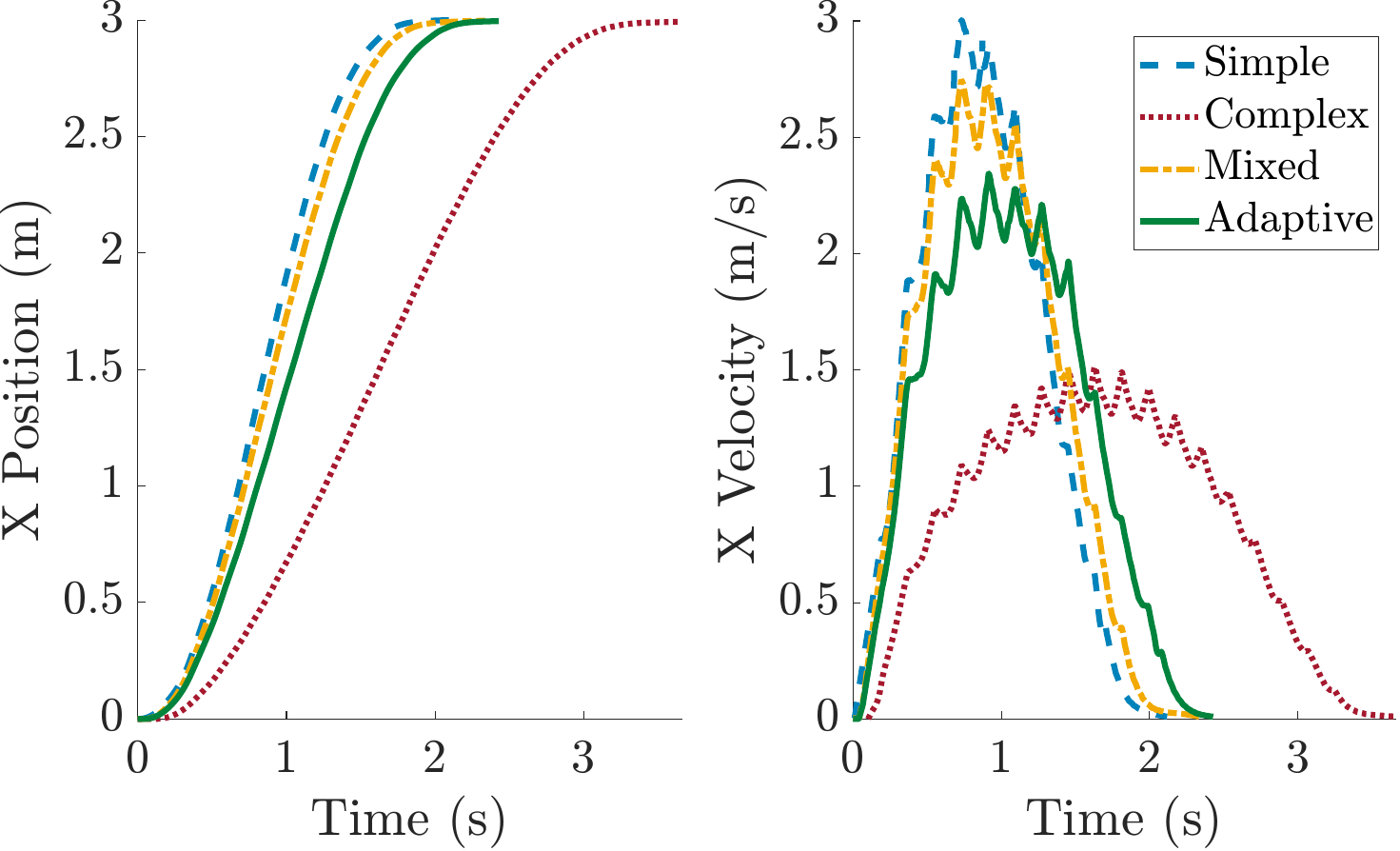}\label{fig:ac:acceleration_experimental_data}}
\hfill
\subfloat[0.485\textwidth][Acceleration experiment solve times. Only the Adaptive configuration changes with problem difficulty as joint constraints are activated and deactivated. Chatter at the end corresponds to entering full support phase before coming to rest.]{\includegraphics[width=0.485\textwidth]{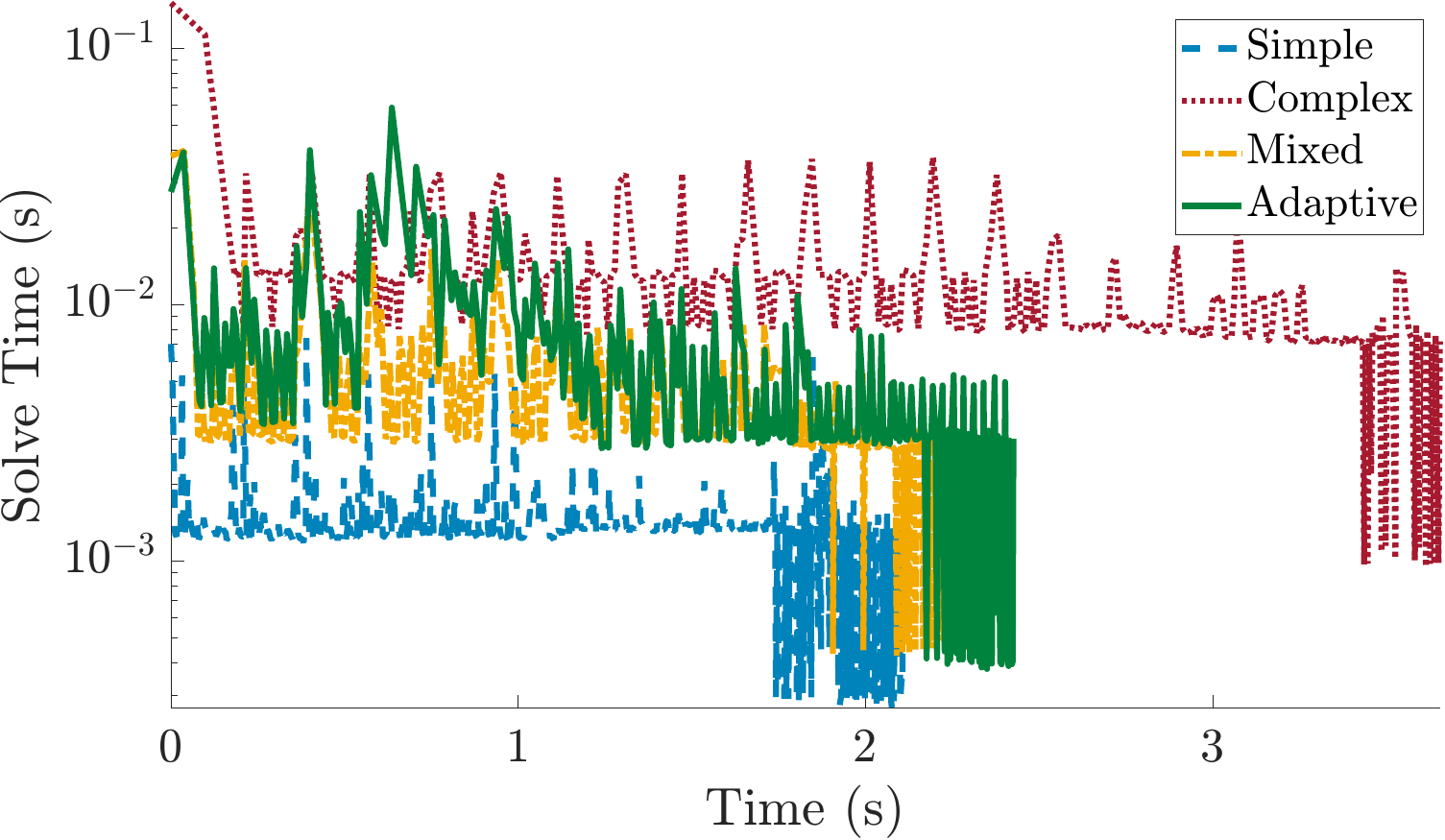}\label{fig:ac:flat_solve_time}}
\quad
\vskip1em
\subfloat[0.485\textwidth][Acceleration experiment simplicity set size, as a percentage of the prediction horizon. The Adaptive configuration shrinks the simplicity set to handle joint constraints but later expands it when feasible.]{\includegraphics[width=0.485\textwidth]{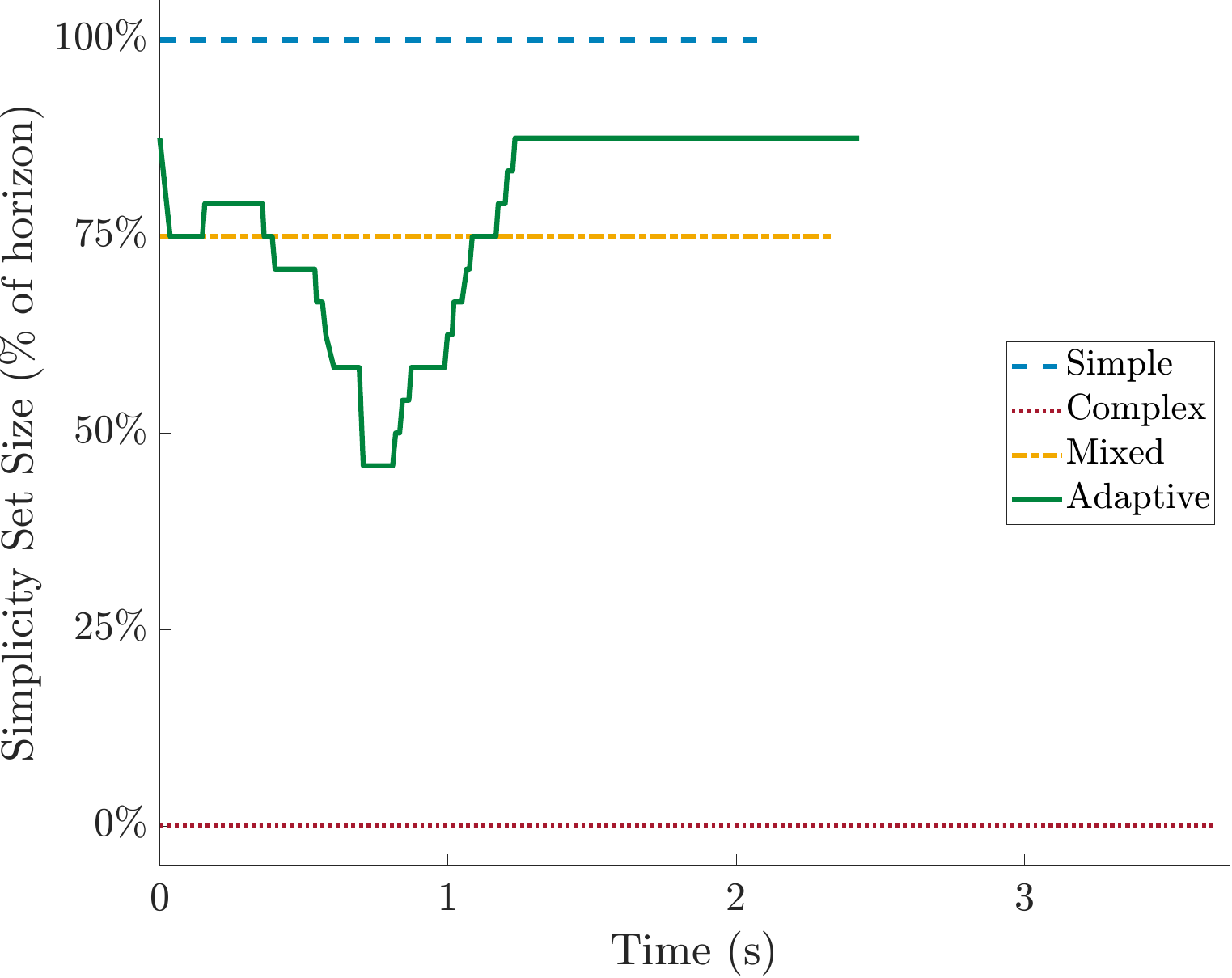}\label{fig:ac:flat_simple_percentage}}
\hfill
\subfloat[0.485\textwidth][Acceleration experiment prediction horizons. Horizons at each time are indicated by horizontal slices of finite elements (dots), where dot color indicates model complexity. Vertical bands correspond to instances where joint constraints (torque or velocity) require additional complexity. ]{\includegraphics[width=0.485\textwidth]{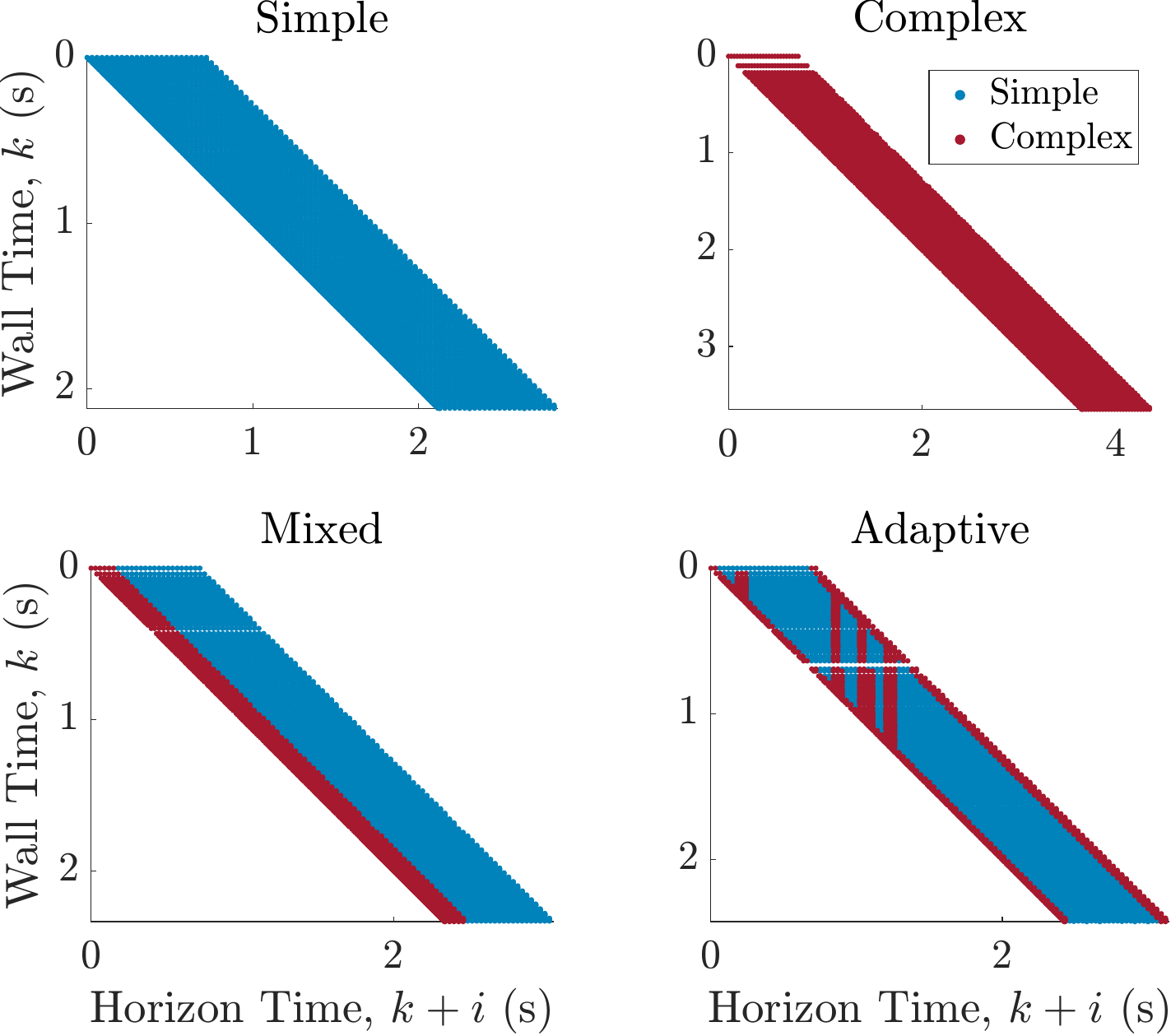}\label{fig:ac:flat_prediction_horizon}}
\quad
\vskip1em
\subfloat[0.485\textwidth][Experimental data for the Acceleration experiment. Bold indicates best observed performance.]{
\renewcommand{\arraystretch}{1.3}
\begin{tabular}{c c c c c}
\hline
Config & Success Rate & \shortstack{Mean Solve \\ Time (ms)} & \shortstack{Slow Solve \\ Rate (\%)}  & \shortstack{Max Velocity \\ (m/s)} \\
\hline
Simple & 10/10 & \textbf{1.2} & \textbf{0.0} & \textbf{3.0} \\
Complex & 10/10 & 11 & 7.4 & 1.5 \\
Mixed & 10/10 & 2.9 & 0.80 & 2.7 \\
Adaptive & 10/10 & 4.0 & 2.1 & 2.3 \\
\hline
\end{tabular}
\renewcommand{\arraystretch}{1}
\label{tab:ac:acceleration_experimental_data}
}
\caption[Data for Acceleration experiment]{Data for Acceleration experiment. The computational efficiency of the Simple configuration permits aggressive commands despite model reductions. Both Adaptive and Mixed outperform Complex because they retain some efficiency.}
\label{fig:ac:acceleration}
\end{figure*}

A series of snapshots of the Adaptive configuration performing the Acceleration task are shown in Fig.~\ref{fig:ac:acceleration_sequence}. Since the peak acceleration of the system occurs at the beginning and end of the trajectory, rapidly converging on a feasible solution to the OCP is essential. Crucially, performing this task does not require exact knowledge of the joint constraints and thus computational efficiency is key.

Results for each configuration are shown in Fig.~\ref{fig:ac:acceleration}, with state trajectories of candidate trials in Fig.~\ref{fig:ac:acceleration_experimental_data}, solve time data in Fig.~\ref{fig:ac:flat_solve_time}, simplicity set data in Figs.~\ref{fig:ac:flat_simple_percentage} and \ref{fig:ac:flat_prediction_horizon}, and quantitative results in Fig.~\ref{tab:ac:acceleration_experimental_data}. The Simple configuration exhibits the best performance with a 97\% increase in top speed over Complex. Mixed performs next best at an 80\% increase in top speed, followed by Adaptive at a 55\% increase in top speed. These reflect the relative complexity of each configuration -- since the Complex system must reason about extraneous constraints over the entire horizon, it takes longer to solve the problem -- shown in Fig.~\ref{fig:ac:flat_solve_time} -- and is thus less capable of stabilizing high-acceleration behaviors. The Simple configuration conversely excels since it is solving a reduced problem. The Mixed and Adaptive systems consist mostly of simple finite elements and thus retain this benefit, although the Adaptive configuration performance is slightly degraded since more complex elements are added at the beginning of the behavior during periods of high joint torque and velocity, as shown in Fig.~\ref{fig:ac:flat_prediction_horizon}. These results support Hypothesis~\ref{hyp:ac:controller_performance} which states that the reducing the model yields improved locomotion performance through more efficient computation.

A baseline version of this experiment was run to remove the effect of solve time on locomotion stability. In this experiment, the simulation update rate was slowed down by a factor of 20, and a maximum effective planning rate of 100 Hz was enforced such that regardless of solve time, new solutions were published at that rate. With these modifications, the maximum velocities were 3.2 m/s for Simple, 2.7 m/s for Complex, 2.9 m/s for Mixed, and 2.7 m/s for Adaptive. All configurations were within 15\%, and constrained by the reach of the swing legs. Notably, the Simple configuration still performed best because it could violate these kinematic constraints in ways that the more complex configurations could not. These effects are discussed in greater depth in Sec.~\ref{sec:ac:discussion}.

\subsection{Step Environment}

\begin{figure*}[!p]
\centering
\subfloat[1.0\textwidth][The Step environment requires navigating kinematic constraints. Snapshots are shown of the trajectories under the Adaptive configuration.]{\includegraphics[width=1.0\textwidth]{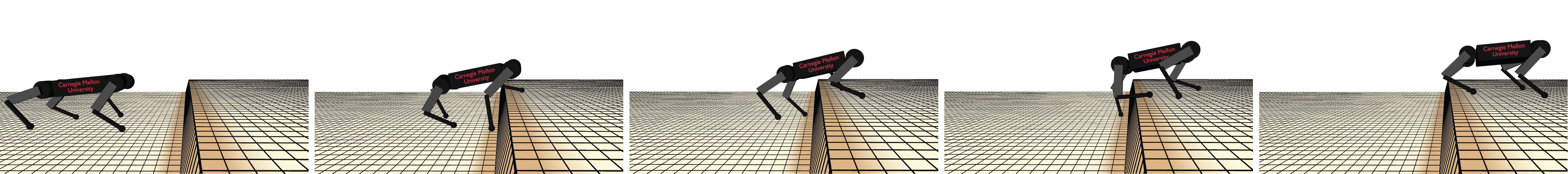}\label{fig:ac:step_20cm_sequence}}
\quad
\vskip1em
\subfloat[0.485\textwidth][Step environment state trajectories. Complex (red) and Adaptive (green) show changes to pitch and yaw (indicated by the arrows) before the Simple (blue) and Mixed (gold) configurations.]{\includegraphics[width=0.485\textwidth]{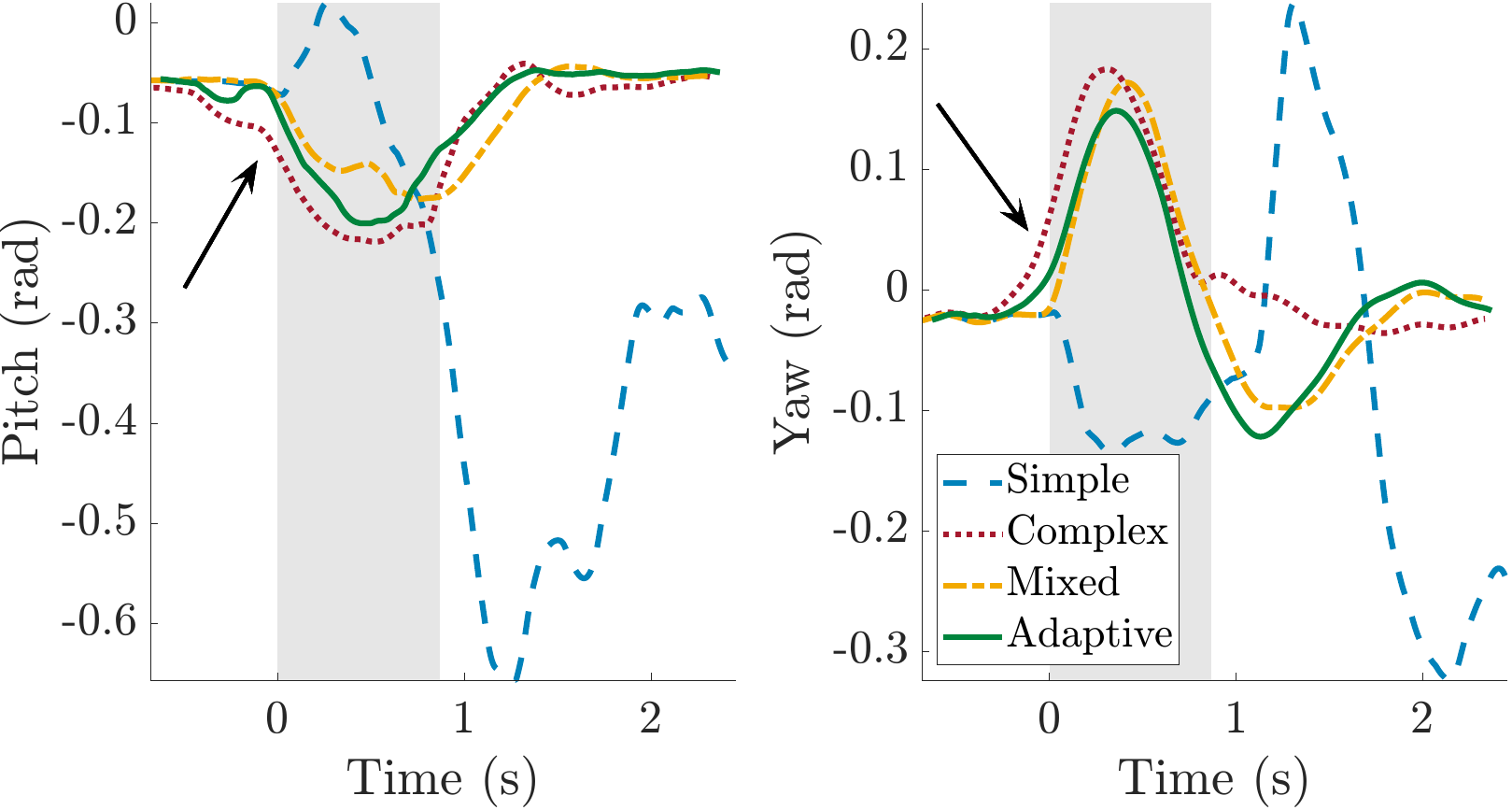}\label{fig:ac:step_20cm_linear_states}}
\hfill
\subfloat[0.485\textwidth][Step environment solve times. The increase four seconds into the behavior corresponds to navigating the step.]{\includegraphics[width=0.485\textwidth]{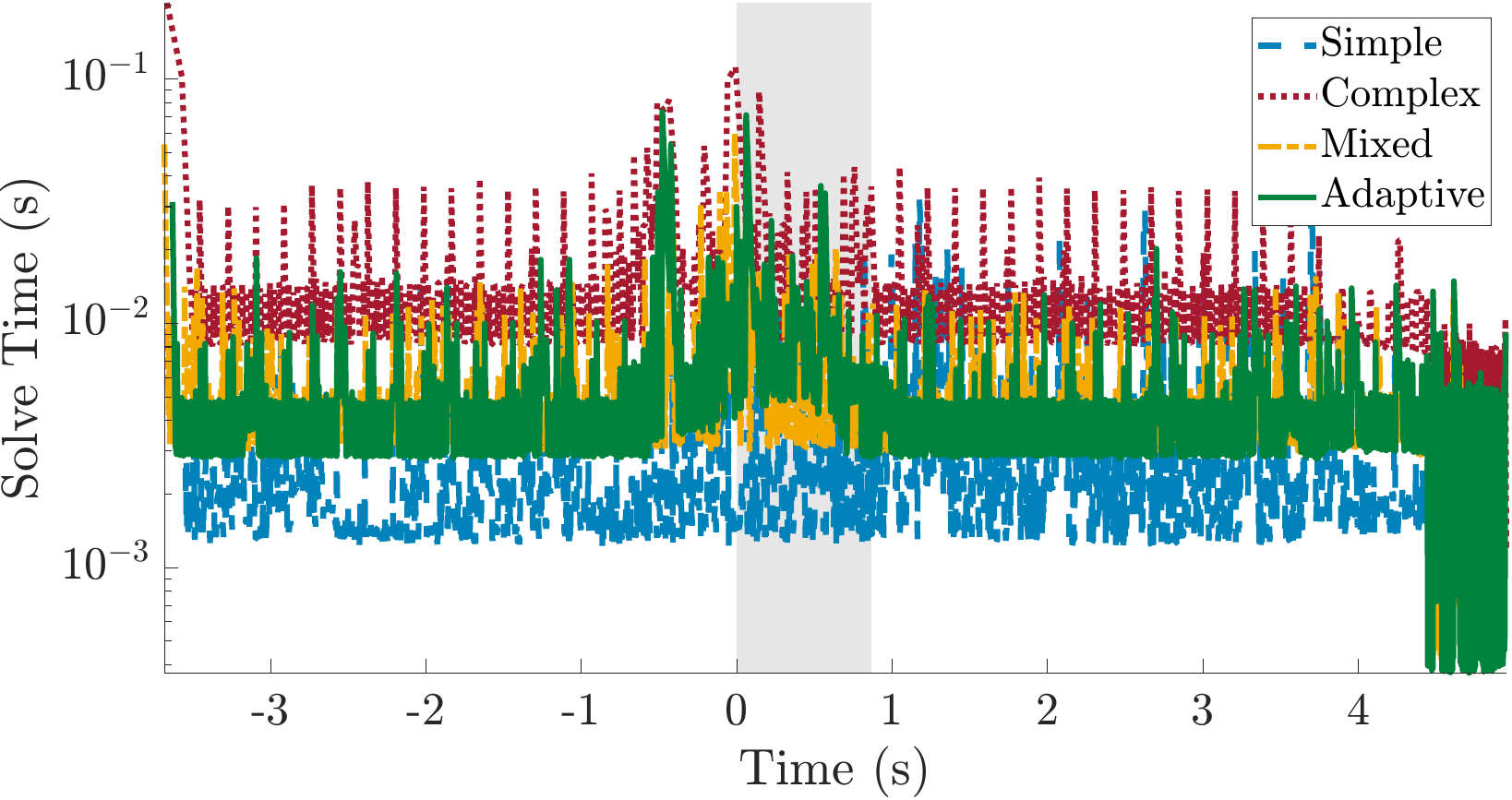}\label{fig:ac:step_20cm_solve_time}}
\quad
\vskip1em
\subfloat[0.485\textwidth][Step environment horizon simplicity set size. The Adaptive configuration is able to simplify the problem for most of the behavior, and quickly recover these simplifications once the difficult behavior is resolved.]{\includegraphics[width=0.485\textwidth]{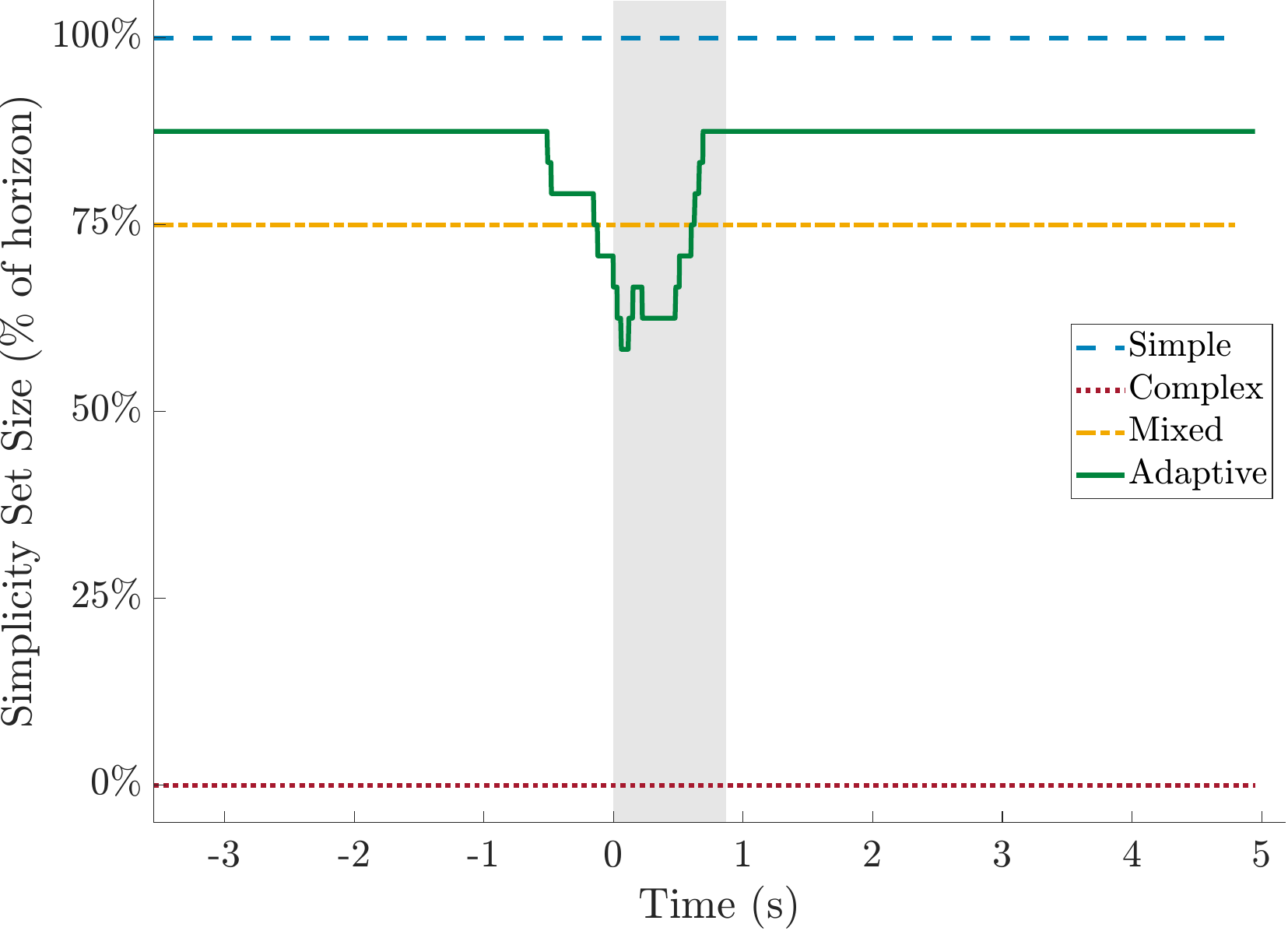}\label{fig:ac:step_20cm_simple_percentage}}
\hfill
\subfloat[0.485\textwidth][Step environment prediction horizons. Horizons at each time are indicated by horizontal slices of finite elements (dots), where dot color indicates model complexity. The vertical bands of increased complexity correspond to instances where joints in the front and back legs approach singularities. ]{\includegraphics[width=0.485\textwidth]{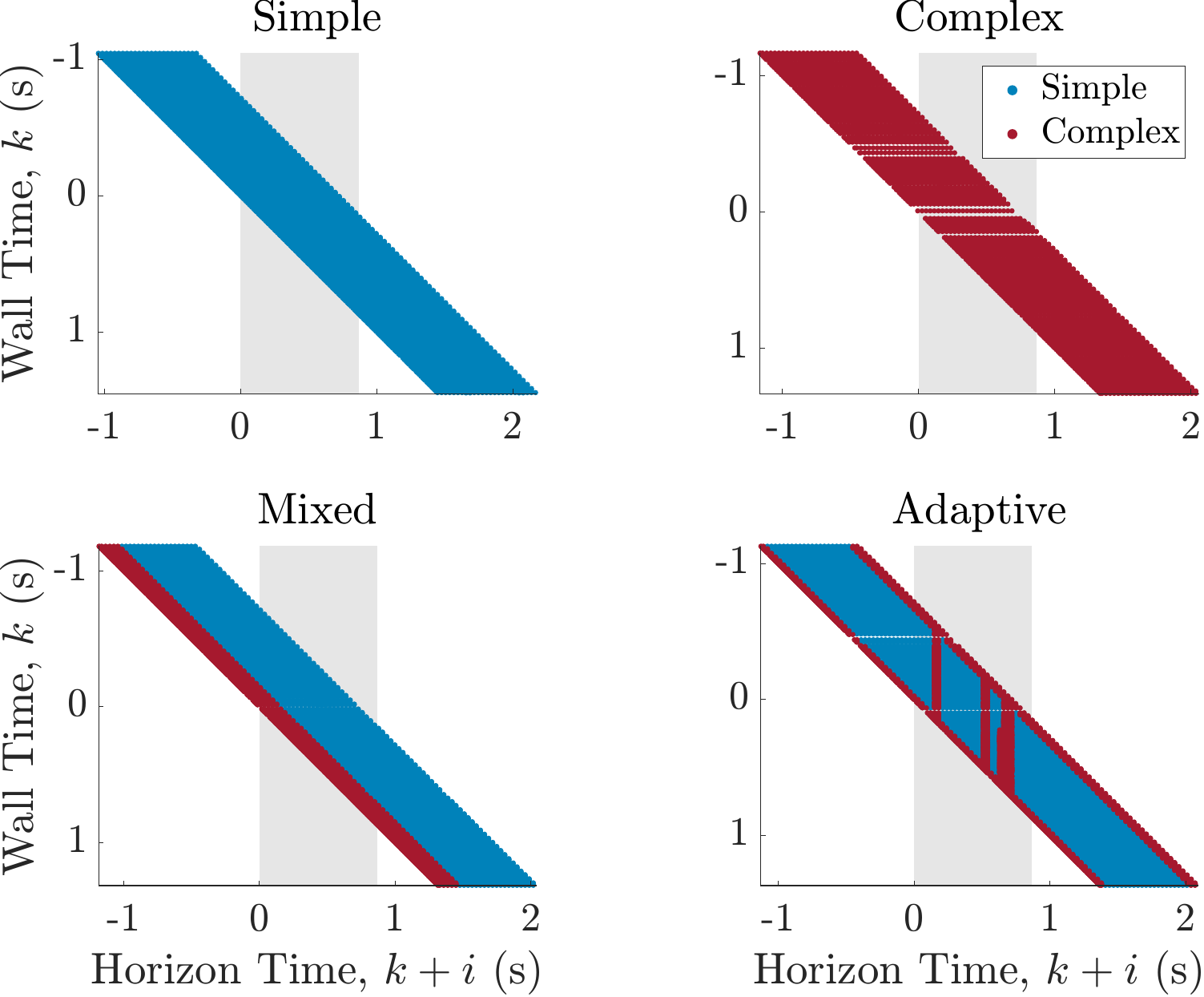}\label{fig:ac:step_20cm_prediction_horizon}}
\quad
\vskip1em
\subfloat[0.485\textwidth][Experimental data for the Step environment.]{
\renewcommand{\arraystretch}{1.3}
\begin{tabular}{c c c c c}
\hline
Config & Success Rate & \shortstack{Mean Solve \\ Time (ms)} & \shortstack{Slow Solve \\ Rate (\%)} & \shortstack{Mean \\ Control (N)}  \\
\hline
Simple & \textbf{10/10} & \textbf{2.2} & \textbf{0.039} & 43  \\
Complex & 9/10 & 11 & 16 & \textbf{16}\\
Mixed & 7/10 & 4.3 & 3.9 & 22\\
Adaptive & \textbf{10/10} & 4.3 & 2.6 & 18 \\
\hline
\end{tabular}
\renewcommand{\arraystretch}{1}
\label{tab:ac:step_experimental_data}
}
\caption[Data for Step environment]{Data for Step environment. The Adaptive configuration is able to leverage admissible reductions for the majority of the behavior while retaining the ability to react quickly to the kinematic constraints required to navigate the step. Gray shading indicates the step-climbing period in the reference trajectory, with $k = 0$ corresponding to the beginning of this period.}
\label{fig:ac:step_20cm}
\end{figure*}

A series of snapshots of the Adaptive configuration navigating the Step environment are shown in the top row of Fig.~\ref{fig:ac:step_20cm_sequence}. The key constraints which must be resolved are the joint limits of the robot and the height of the toe, as the system must ensure the toe clears the step while also ensuring the rear legs can still reach the terrain for support.

Results from each MPC configuration are shown in Fig.~\ref{fig:ac:step_20cm} and quantitatively summarized in Table~\ref{tab:ac:step_experimental_data}. The state trajectories are shown in Fig.~\ref{fig:ac:step_20cm_linear_states}. All of the configurations are able to complete the task and reach the goal at least some of the time, although the manner in which the task is completed differs. The lack of constraint information in the Simple configuration and the myopia of the Mixed configuration often result in foot scuffs when crossing the step, which nearly destabilize the system and require large control actions to correct. The Complex and Adaptive configurations are able to identify potential constraint violations caused by the step sooner and react by increasing the walking height and rotating the body to more safely navigate the step (arrows in Fig.~\ref{fig:ac:step_20cm_linear_states}).

The computational effort of each configuration is shown in Fig.~\ref{fig:ac:step_20cm_solve_time}. Unsurprisingly the Simple configuration is consistently the fastest since its model is the most sparse and it is unaware of the nonlinear joint kinematic constraints, while the Complex configuration consistently takes the longest, especially to find an initial solution and also once it sees the step. Both the Mixed and Adaptive configurations have intermediate nominal solve times, but differ when the step approaches. The Adaptive configuration immediately takes much longer to solve the problem as it needs to reason about this new information, but once a valid solution is found it settles back to its nominal solve time as more complex elements are converted back to simple ones. Meanwhile the Mixed formulation only increases in solve time when the step is within its shorter window of complex elements, and planning the large control forces required to navigate the step on such short notice causes a large and sustained increase in solve time.

The degree of horizon simplification is shown in Fig.~\ref{fig:ac:step_20cm_simple_percentage}. While the fixed-complexity configurations remain uniform for the entire task, the Adaptive configuration clearly leverages additional complexity when encountering the step. However, even in the worst case around half of the horizon remains simplified, the effects of which are seen in the lower solve times compared to the Complex configuration. These horizons are visually shown in Fig.~\ref{fig:ac:step_20cm_prediction_horizon}, which shows the prediction horizon at each time with complex and simple elements distinguished by different colors. The Adaptive configuration clearly reacts to the step -- as elements in the terminal region require leaving the simple manifold, adaptive complexity MPC fills in new elements with additional complexity. When the terminal region returns to the simple manifold, the algorithm recognizes this and allows simple elements back into the horizon. Together, these results support Hypothesis~\ref{hyp:ac:simple_frequency} which states that reductions are frequently admissible for candidate terrains, and Hypothesis~\ref{hyp:ac:controller_performance} which states that capturing the complex dynamics and constraints expands the range of executable tasks.

\subsection{Gap Environment}

\begin{figure*}[!p]
\centering
\subfloat[1.0\textwidth][The Gap environment requires navigating both kinematic and dynamic constraints. Snapshots are shown of the trajectories under the Adaptive configuration.]{\includegraphics[width=1.0\textwidth]{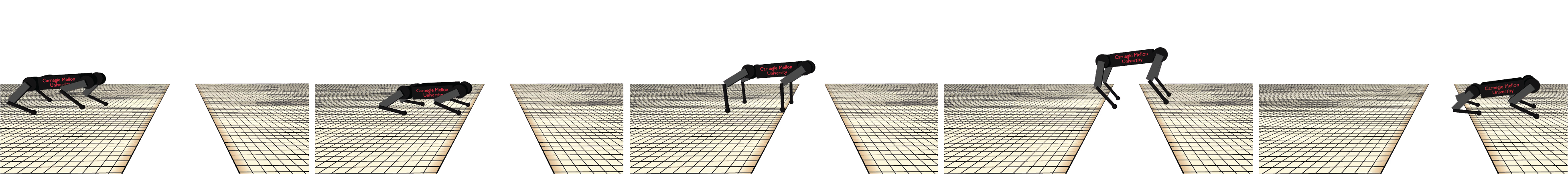}\label{fig:ac:gap_40cm_sequence}}
\quad
\vskip1em
\subfloat[0.485\textwidth][Gap environment state trajectories. Complex (red) and Adaptive (green) show changes to horizontal velocity and vertical position (indicated by the arrows) before the Simple (blue) and Mixed (gold) configurations.]{\includegraphics[width=0.485\textwidth]{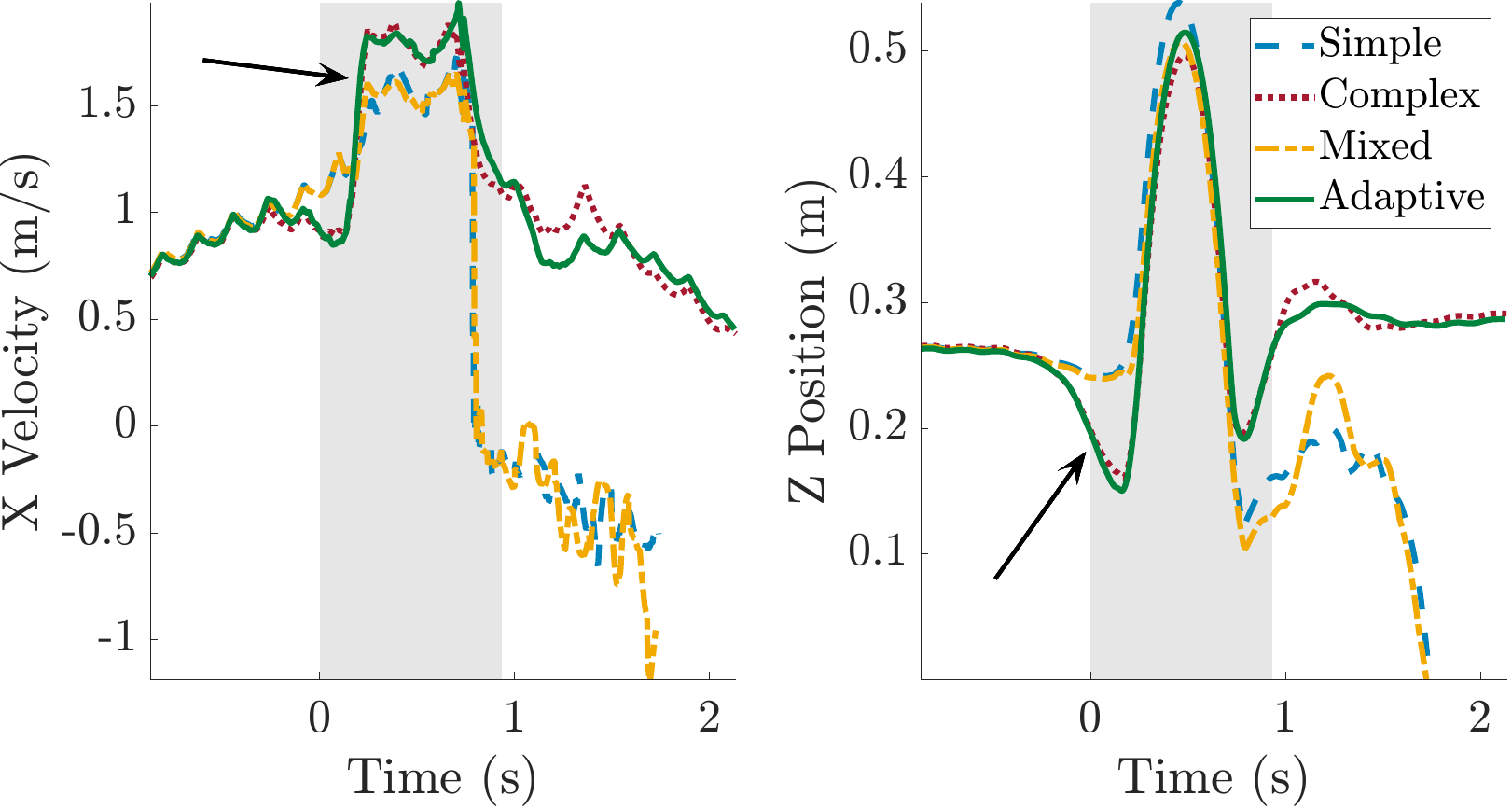}\label{fig:ac:gap_40cm_linear_states}}
\hfill
\subfloat[0.485\textwidth][Gap environment solve times. The increase four seconds into the behavior corresponds to navigating the gap. The sustained increases for the Simple and Mixed configurations correspond to failed solves after unsuccessful landing.]{\includegraphics[width=0.485\textwidth]{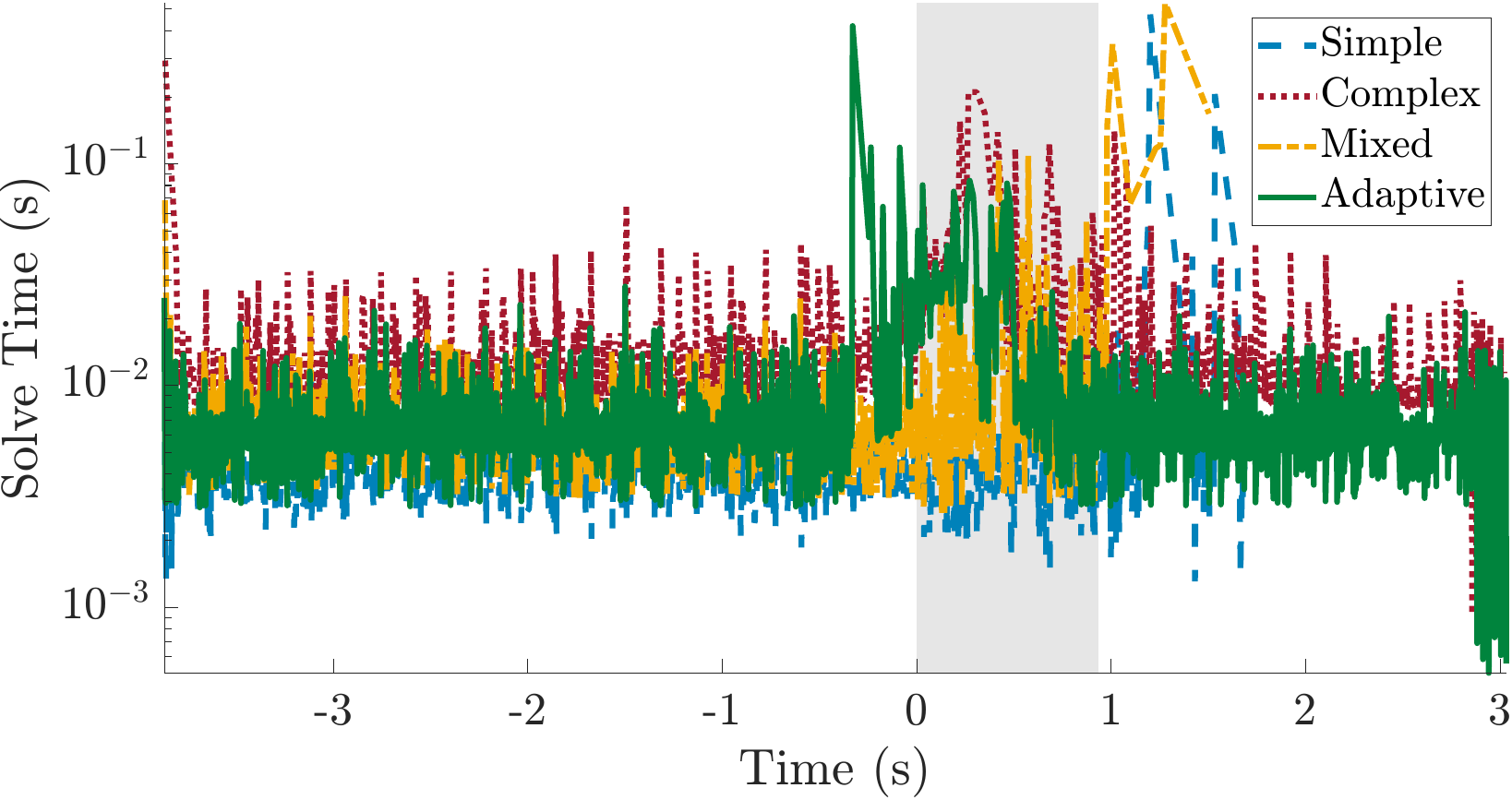}\label{fig:ac:gap_40cm_solve_time}}
\quad
\vskip1em
\subfloat[0.485\textwidth][Gap environment horizon simplicity set size. Similarly to the Step environment, the Adaptive configuration is able to simplify most of the horizon, with the most complexity occurring when both takeoff and touchdown are within the horizon.]{\includegraphics[width=0.485\textwidth]{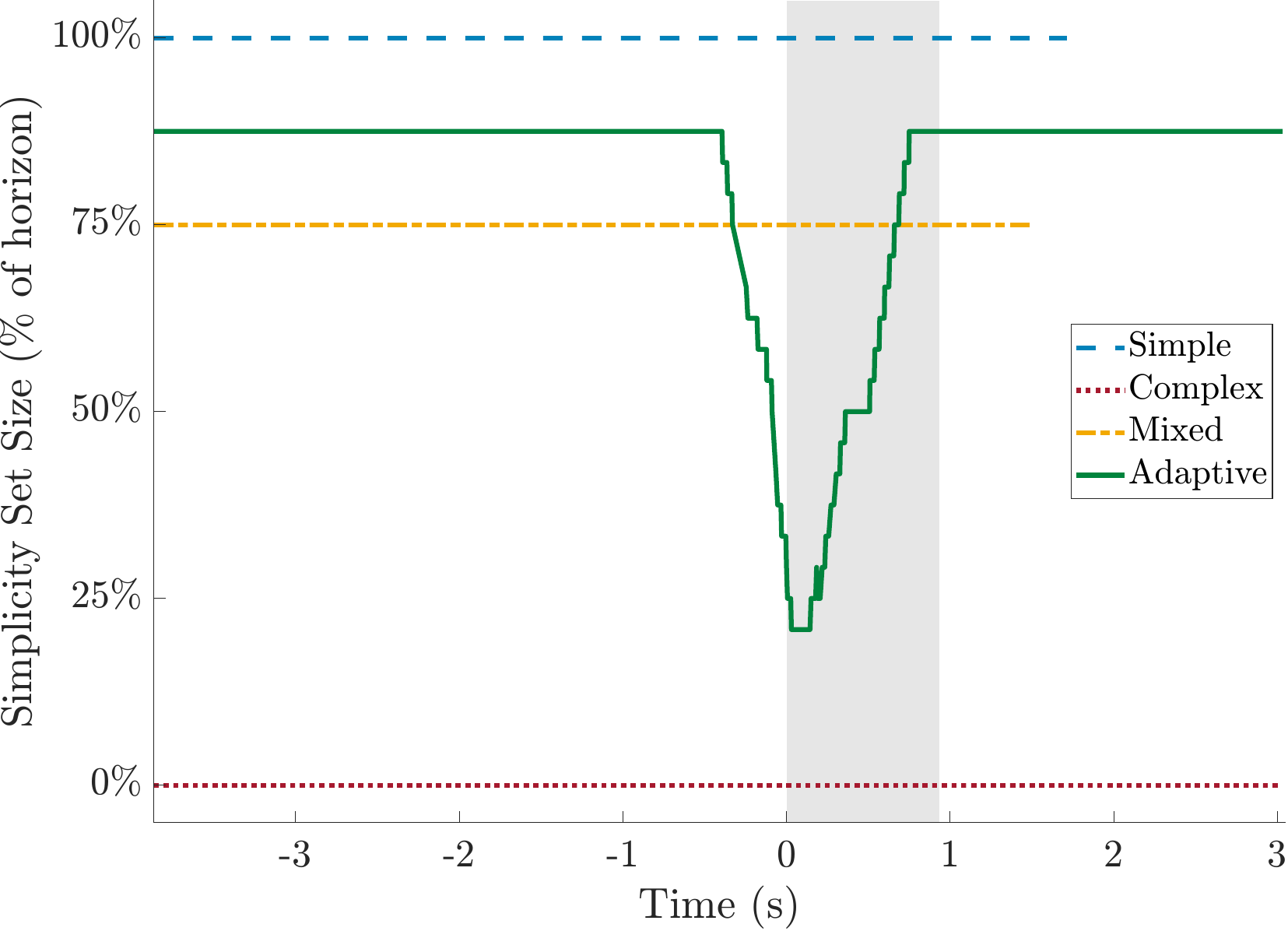}\label{fig:ac:gap_40cm_simple_percentage}}
\hfill
\subfloat[0.485\textwidth][Gap environment prediction horizons. Horizons at each time are indicated by horizontal slices of finite elements (dots), where dot color indicates model complexity. The vertical bands of increased complexity correspond to takeoff and touchdown, gaps indicate large solve times.]{\includegraphics[width=0.485\textwidth]{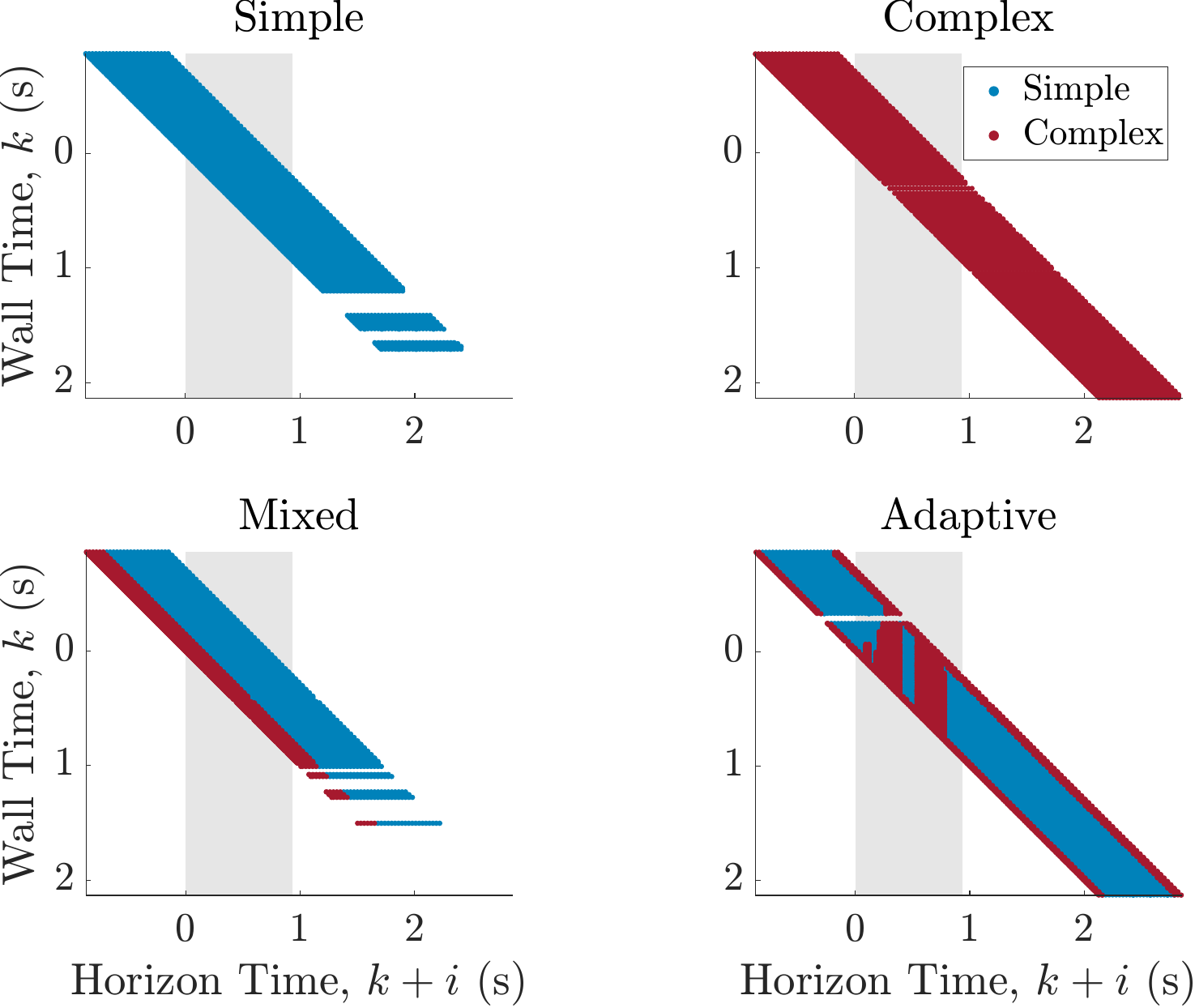}\label{fig:ac:gap_40cm_prediction_horizon}}
\quad
\vskip1em
\subfloat[0.485\textwidth][Experimental data for the Gap environment.]{
\renewcommand{\arraystretch}{1.3}
\begin{tabular}{c c c c c}
\hline
Config & Success Rate & \shortstack{Mean Solve \\ Time (ms)} & \shortstack{Slow Solve \\ Rate (\%)}  & \shortstack{Mean \\ Control (N)} \\
\hline
Simple & 0/10 & -- & -- & --  \\
Complex & 9/10 & 13 & 13 & 33\\
Mixed & 0/10 & -- & -- & --\\
Adaptive & \textbf{10/10} & \textbf{7.5} & \textbf{4.3} & \textbf{31} \\
\hline
\end{tabular}
\renewcommand{\arraystretch}{1}
\label{tab:ac:gap_experimental_data}
}
\caption[Data for Gap environment]{Data for Gap environment. The Adaptive and Complex configurations are able to reason about constraints at the end of the horizon, allowing them to alter the leap to increase forward velocity and successfully land. Gray shading indicates the leap period in the reference trajectory, with $k = 0$ corresponding to the beginning of this period.}
\label{fig:ac:gap_40cm}
\end{figure*}

A series of snapshots of the Adaptive configuration navigating the Gap environment are shown in Fig.~\ref{fig:ac:gap_40cm_sequence}. Due to the state-dependent actuator limits which reduces peak torque as joint velocity increases, the system must accelerate early to ensure enough velocity to land safely on the other side. Joint kinematics also limit how far forward the legs can reach to prepare for landing, so the system must be aware of these constraints before takeoff to ensure sufficient controllability.

Results for each configuration in the Gap environment are shown in Figs.~\ref{fig:ac:gap_40cm} and quantitatively summarized in Table~\ref{tab:ac:gap_experimental_data}. Only the Adaptive configuration has both the efficiency and model fidelity required to solve this task reliably. The Complex configuration fails in the majority of the trials due to its excessive computational effort, while the Simple and Mixed configurations never succeed because they lack the requisite constraint knowledge and thus do not leap far enough. Like in the Step environment, the configurations which are able to recognize constraints near the end of the horizon (Adaptive and Complex) do so immediately by lowering towards the ground to obtain a longer leaping stroke and accelerating forwards to ensure enough velocity to reach the other side of the gap, as shown in Fig.~\ref{fig:ac:gap_40cm_linear_states}.

The solve times for the Gap environment shown in Fig.~\ref{fig:ac:gap_40cm_solve_time} demonstrate similar trends as during the Step environment. Both Simple and Mixed formulations maintain their performance until they fail to cross the gap, which quickly causes failed solves. Both Complex and Adaptive configurations have periods of longer solves to plan the leaping and landing phase, but the Adaptive formation is able to recover faster solve times sooner due to its ability to convert complex element to simple ones near the end of the leap. This is further supported by the data in Table~\ref{tab:ac:gap_experimental_data} which shows a threefold reduction in Adaptive solve times exceeding one timestep compared to the Complex configuration. The mechanism for this reduction is further illustrated in the simplification percentages shown in Fig.~\ref{fig:ac:gap_40cm_simple_percentage} and the prediction horizons shown in Fig.~\ref{fig:ac:gap_40cm_prediction_horizon}. Even in the worst-case portion of the behavior, the Adaptive configuration retains 25\% simplification of the horizon, and once more feasible elements begin entering the horizon the Adaptive configuration can take advantage of the reduced complexity to improve solve times. These results support Hypotheses \ref{hyp:ac:simple_frequency} and \ref{hyp:ac:controller_performance}.

\section{Discussion} \label{sec:ac:discussion}

Table~\ref{tab:ac:exp_summary} summarizes the key results of the each of the experiments. While no one configuration is optimal for all three tasks, only the Adaptive configuration is able to retain the benefits of both model complexity and computational efficiency across each task. Among the configurations which can solve the difficult Gap environment and avoid stumbling over the Step environment, it performs the best at the highly dynamic Acceleration task due to its ability to more quickly react to aggressive commands. This suggests that the ideal scenario for deploying ACMPC rather than fixed configurations -- whether simple or complex -- is one with a high degree of heterogeneity, where nominal operation is generally expected but occasionally interrupted by the need for foresight or reactivity. This describes many common tasks and environments for legged locomotion, such as walking over an open stretch of flat pavement before scrambling over a curb or dodging away from a moving obstacle.

The computational efficiency afforded by ACMPC could have broader benefits for many robotic systems beyond the increase in agility measured here. Even with modern architectures, many industrial mobile robots -- legged systems in particular -- have tight computational, thermal, and power budgets. Algorithms which use fixed complexity configurations are unable to regulate their own draw on these valuable resources, and those that use highly complex models such as those discussed in Sec.~\ref{sec:ac:related_work} may unnecessarily consume these resources during trivial tasks like walking over flat terrain. Applying the concept of adaptive complexity to those methods would allow resources to be allocated elsewhere or even invested back into the controller by extending the prediction horizon.

This additional performance does not come without limitation. Adaptive complexity MPC is subject to the formulation of the complex system, and in particular the structure of its dynamics and constraints. Introducing additional numerical complexity such as hard constraints or non-convexity can make solving the OCP more susceptible to over-constrained problems, local minima, or poor convergence rates, which are then transferred to the adaptive configuration. This is most notable in the Step environment, in which the Simple configuration demonstrated remarkable ability to complete the task without constraint knowledge due to its numerical robustness. Ongoing work into well-conditioned OCP formulations such as the QP approximation approach of \cite{grandia2022perceptive} or methods which identify which constraints are most necessary would benefit the approaches discussed here by removing the potentially disproportionate effect of a few complex finite elements.

Another current drawback of adaptive complexity is robustness to unexpected errors in the simplicity set caused by model mismatch or disturbances. Introducing additional complexity in the interior of the horizon can degrade the initialization of the OCP -- this is largely why the Mixed configuration outperformed Adaptive in the Acceleration task. This effect could potentially be alleviated by recent approaches which warm-start the OCP with experiential data \cite{mansard2018using}, or possibly avoided by applying robust MPC techniques \cite{bemporad1999robust}. Conversely, one approach employed in this work to protect against this effect was to avoid re-adding elements to the simplicity set once they had been removed. While this reduced the instances of interior elements unexpectedly requiring additional complexity, it resulted in overly conservative simplicity sets -- many elements which initially violated the admissibility conditions later satisfied them once the controller took corrective actions. It may be possible to safely generate more optimistic simplicity sets, which could enable real-time deployment for more difficult tasks.

In addition, many hierarchical systems leverage reduced-order models to generate reference trajectories entirely in the simple system, making infeasible references highly relevant. Investigations into adapting model complexity to handle infeasible references such as \cite{batkovic2021model} or other ways to satisfy Assumption~\ref{as:ac:psi_dagger_reduces_cost} without requiring a reference would be crucial to expand this work to a broader class of systems. However, it not uncommon to define a simple system with the same dynamics as the complex system but with simpler constraints. This would permit an identity operator for the lifting function without need for a reference, yielding a simpler OCP which only evaluates constraints where necessary. This approach is akin to lazy strategies used in motion planning \cite{bohlin2000path}.

\begin{table}[]
    \centering
    \caption{Summary of experimental results in key metrics.}
    \label{tab:ac:exp_summary}
    \begin{tabular}{c c c c}
        \hline
        Task & Acceleration & Step & Gap  \\
        Key Metric & Max Velocity (m/s) & Mean Control (N) & Success Rate \\
        \hline
        Simple & \textbf{3.0} & 43 & 0/10 \\
        Complex & 1.5 & \textbf{16} & 9/10 \\
        Mixed & 2.7 & 22 & 0/10 \\
        Adaptive & 2.3 & 18 & \textbf{10/10} \\
        \hline
    \end{tabular}
\end{table}

\section{Conclusion} \label{sec:ac:conclusion}
This work presents a formulation of adaptive complexity MPC which actively identifies regions where dynamics and constraints can be simplified without compromising the feasibility or stability of the original system. Analysis of the proposed approach demonstrates that under key conditions these simplifications do not compromise the stability properties of the original system, and can enable new behaviors by acting quickly to perform agile motions or looking further into the future to execute behaviors. These advantages are demonstrated on a simulated quadrupedal robot performing agile behaviors with challenging environmental constraints, and in particular expanding the leaping capability through receding horizon execution with knowledge of joint constraints.

While the MPC formulation presented here was primarily evaluated in locomotion applications, future work could investigate its applicability to other domains that employ hierarchical structures, such as manipulation. For example, often in manipulation settings the internal joints of the manipulator are neglected and planning is primarily conducted in the space of object motions and forces. Adaptive complexity MPC would enable an efficient handling of manipulator kinematics only when necessary so that the system can respect these constraints while largely retaining the benefits of improved efficiency, including faster reactions to unexpected object motion or longer planning horizons.

\appendix

\subsection{Proofs Required for Adaptive Complexity MPC Stability} \label{sec:ac:appendix_stability_proofs}

The proof of Proposition~\ref{th:ac:recursive_feasibility} shows recursive feasibility by finding a feasible solution to the ACOCP at the successor state.
\begin{proof} [Proof (Proposition~\ref{th:ac:recursive_feasibility})]
Since $x^c_{k} \in \mathcal{X}_{N^a}$, there exists a solution $\mathbf{u}^{*a}(x^c_{k})$ to $\text{ACOCP}(x^c_{k})$. Because $\mathbf{u}^{*a}(x^c_{k})$ is a solution of (\ref{eq:ac:adaptive_ocp}), by Proposition~\ref{th:lifting_terminal_state} the corresponding predicted terminal state $x_{T} \in \mathcal{X}^c_{T}$. Let the successor control sequence $\Tilde{\mathbf{u}}^a(x^c_{k})$ be defined by (\ref{eq:ac:successor_control_seq}). We claim this sequence is feasible for $\text{ACOCP}(x^c_{k+1})$ solved at successor state $x^c_{k+1} \coloneqq f^c_{\textnormal{acmpc}}(x^c_k)$.

Firstly, the controls $u^{*a}_1, \dots, u^{*a}_{N^a-1}$ which are elements of $\mathbf{u}^{*a}(x^c_{k})$ which was a solution to (\ref{eq:ac:adaptive_ocp}), all lie in $\mathcal{Z}^{a}_{f}$ by Proposition~\ref{th:lifting_feasibility}. It follows from Assumption~\ref{as:ac:original_mpc_conditions} (since $x_{T} \in \mathcal{X}^c_{T}$) that $u_{T}(x_{T}) \in \mathcal{Z}^c_{f}$, and thus every element of $\Tilde{\mathbf{u}}^{a}(x^c_{k})$ satisfies the control constraint of (\ref{eq:ac:adaptive_ocp}).

Next we consider the state constraint. By Proposition~\ref{th:lifted_state} the state sequence resulting from initial state $f^c_{\textnormal{acmpc}}(x^c_{k})$ and control sequence $\Tilde{\mathbf{u}}^{a}(x^c_k)$ is $\Tilde{\mathbf{x}}^{a} \coloneqq \left[ \Tilde{x}^{a}_0, \Tilde{x}^{a}_1, \dots, \Tilde{x}^{a}_{N^a} \right]$ where,
\begin{align}
    \Tilde{x}^{a}_{i} &= x^{*{a}}_{i+1}, \quad i = 0, \dots, N^a-1 \\
    \Tilde{x}^{a}_{N^a} &= f^c_{u_{T}}(x_{T})
\end{align}
and $\Tilde{x}^{a}_0 = x^{*{a}}_1 = f^c(x^c_{k}, u^{*l}_{0}) = f^c_{\textnormal{acmpc}}(x^c_{k})$. By Proposition~\ref{th:lifting_feasibility}, the states $x^{*{a}}_1, \dots, x_{T}$ satisfy the state constraint. Since $x_{T} \in \mathcal{X}^c_{T}$, Assumption~\ref{as:ac:original_mpc_conditions} implies that $f^c_{u_{T}}(x_{T}) \in \mathcal{X}^c_{T} \subset \mathcal{X}_{N^{a}}$, so that every element of the state sequence $\Tilde{\mathbf{x}}^{a}$ satisfies the state constraint, and the new terminal state $\Tilde{x}^{a}_{N^a} = f^c_{u_{T}}(x_{T})$ satisfies the stability constraint. Hence $\Tilde{\mathbf{u}}^{a}(x^c_{k})$ is feasible for $\text{ACOCP}(f^c_{\textnormal{acmpc}}(x^c_{k}))$ and $f^c_{\textnormal{acmpc}}(x^c_{k}) \in \mathcal{X}_{N^a}$.
\end{proof}

The proof of Proposition~\ref{th:ac:decreasing_cost} shows that the OCP cost function decreases along the closed loop system by leveraging shared terms in the solutions along with the observation that the cost of a state in the simple system is equal to the cost of that state lifted into the complex system.
\begin{proof} [Proof (Proposition~\ref{th:ac:decreasing_cost})]
The sequence pairs $(\mathbf{u}^{*a}(x^c_{k}), \Tilde{\mathbf{u}}^a(x^c_{k}))$ and $(\mathbf{x}^{*a}, \Tilde{\mathbf{x}}^a)$ have common elements and thus the cost sequence can be simplified,
\begin{align} \label{eq:ac:cost_decreasing_proof}
    V^{*a}_{N^a}(f^c_{\textnormal{acmpc}}(x^c_{k})) -& V^{*a}_{N^a}(x^c_{k})  \nonumber \\
    \leq& V^{a}_{N^a}(x^c_{k+1}, \Tilde{\mathbf{u}}^a(x^c)) - V^{a}_{N^a}(x^c_{k}, \mathbf{u}^{*a}(x^c_{k})) \nonumber \\
    =& \big(V_{I}^c(x_{T}, u_{T}(x_{T})) + V_{T}(f_{u_{T}}(x_{T}))\big) \nonumber \\
    & \quad - \big(V_{I}^c(x^c_{k}, u^c_{\textnormal{acmpc}}(x^c_{k})) + V_{T}(x_{T})\big)  \nonumber \\
    & \quad + \big(V_{I}^c(x^c_{k+1}, u^{*a}_1) - V_{I}^a(x^c_{k+1}, u^{*a}_1)\big)
\end{align}
From the definition of the adaptive cost in (\ref{eq:ac:adaptive_cost}) and Assumption~\ref{as:ac:psi_dagger_reduces_cost}, $V_{I}^c(x^c_{k+1}, u^{*a}_1) - V_{I}^a(x^c_{k+1}, u^{*a}_1) \geq 0$. Since $x_{T} \in \mathcal{X}^c_{T}$, Assumption~\ref{as:ac:original_mpc_conditions} implies,
\begin{align}
    V_{T}(f_{u_{T}}(x_{T})) - V_{T}(x_{T}) \leq - V_{I}^c(x_{T}, u_{T}(x_{T}))
\end{align}
Hence every term on the right side of \eqref{eq:ac:cost_decreasing_proof} is positive semi-definite, and thus (\ref{eq:ac:decreasing_cost}) is satisfied for every $x^c_{k} \in \mathcal{X}_{N^a}$.
\end{proof}

The proof of Theorem~\ref{th:ac:adaptive_complexity_stability} follows from the prior propositions:
\begin{proof} [Proof (Theorem~\ref{th:ac:adaptive_complexity_stability})]
The inequalities in (\ref{eq:ac:stability_inequality}) follow from the structure of the value function defined in Assumption~\ref{as:ac:original_mpc_conditions}. The descent inequality in (\ref{eq:ac:stability_decreasing_cost}) is given by the structure of the value function along with Proposition~\ref{th:ac:decreasing_cost}. Asymptotic stability of the origin with region of attraction $\mathcal{X}_{N^a}$ follows from standard Lyapunov theory  \cite[Appendix B]{rawlings2017model}.
\end{proof}

\subsection{Proof of admissibility conditions for legged system} \label{sec:ac:legged_conditions_proof}
\begin{proof} [Proof (Lemma~\ref{th:ac:legged_conditions})]
The conditions under which the legged system described in Section~\ref{sec:ac:application_to_legged_systems} can be admissibly simplified rely on the feasibility of the reference trajectory. Lemma~\ref{th:ac:legged_conditions} states that an index within a lifted trajectory of the legged system can be admissibly reduced if the lifted components of the state-control pair $q_{\text{foot}}, \dot{q}_{\text{foot}}$, and $u_{\text{foot}}$ all lie on the trajectory and satisfy the constraints in \eqref{eq:ac:legged_complex_constraints}. We proceed by each condition for admissibility defined in \eqref{eq:ac:admissibility_conditions}, noting that \eqref{eq:ac:admissibility_terminal_states} is trivially satisfied by the assumptions of the lemma.

The condition in \eqref{eq:ac:admissibility_feasible} requires that constraints would be satisfied if the system were reduced, i.e.~$\psi^\dagger \circ \psi (z^l_i) \in \mathcal{Z}^c$. By the conditions given on the values of $z^l_i$ in the null space of $\psi$ and the definition of $\psi^\dagger$, it follows that $\psi^\dagger \circ \psi(z^l_i) = z^l_i$. Since $z^l_i \in \mathcal{Z}^c$, it follows that $\psi^\dagger \circ \psi(z^l_i) \in \mathcal{Z}^c$.

The condition in \eqref{eq:ac:admissibility_anchor} requires that the complex system is exactly anchored by the simple system at that index in the trajectory, i.e.~$\psi^\dagger_x \circ f^s \circ \psi(z^l_i) = f^c(z^l_i)$. We show this by directly applying the dynamics and mappings in Section~\ref{sec:ac:application_to_legged_systems}, dropping the index $k$ for simplicity,
\begin{align}
    \psi^\dagger_x \circ f^s \circ \psi(z^l) &=
    \psi^\dagger_x \circ \renewcommand{\arraystretch}{1.2}
    \left[ \begin{array}{c}
        \dot{q}_{\text{lin}} \\
        J_\omega(q_{\text{ang}}) \omega \\
        \frac{1}{m}\sum_j^n u_{\text{body},j} - g \\
        W(q_{\text{lin}}, \bar{q}_{\text{foot}}, \omega, u_{\text{body}})
    \end{array} \right] \nonumber
    \\
    &= \left[ \begin{array}{c}
        \dot{q}_{\text{lin}} \\
        J_\omega(q_{\text{ang}})\omega \\
        \dot{\bar{q}}_{\text{foot}} \\
        \frac{1}{m}\sum_j^n u_{\text{body},j} - g \\
        W(q_{\text{lin}}, \bar{q}_{\text{foot}}, \omega, u_{\text{body}}) \\
        \bar{u}_{\text{foot}}
    \end{array} \right] \nonumber
    \\
    &= f^c(z^l)
\renewcommand{\arraystretch}{1} \nonumber
\end{align}
Lastly, the condition in \eqref{eq:ac:admissibility_enter_manifold} requires that the dynamics at the prior state lead to the manifold, i.e.~$\psi^\dagger_x \circ \psi  \circ f^c(z^l_{i-1}) = f^c(z^l_{i-1})$. Since the trajectory $z^l$ is valid for the system in \eqref{eq:ac:legged_complex_dynamics}, $x^l_{i} = f^c(z^l_{i-1})$. Additionally, since $\psi^\dagger \circ \psi(z^l_i) = z^l_i$, it follows that $\psi^\dagger_x \circ \psi(z^l_i) = x^l_i$. Thus $\psi^\dagger_x \circ \psi  \circ f^c(z^l_{i-1}) = f^c(z^l_{i-1})$.
\end{proof}

\section*{Acknowledgments}
The authors would like to thank Lorenz T.~Biegler and Maxim Likhachev for insightful discussions on the theoretical properties of MPC and other motion planning algorithms.

\bibliography{main}

\begin{IEEEbiography}[{\includegraphics[width=1in,height=1.25in,clip,keepaspectratio]{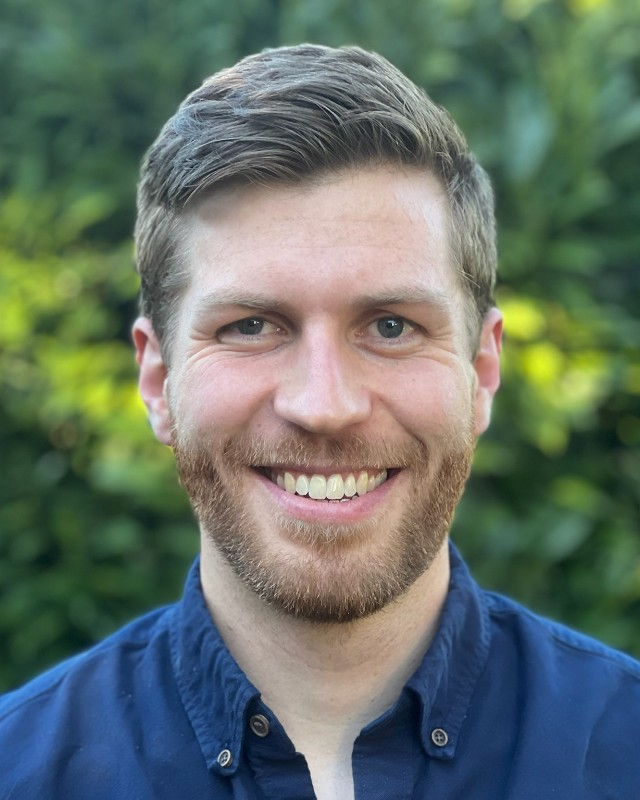}}]{Joseph Norby}
(Member, IEEE) received the B.S.~degree in mechanical engineering from the University of Notre Dame, Notre Dame, IN, USA in 2016, and the Ph.D.~degree in mechanical engineering from Carnegie Mellon University, Pittsburgh, PA, USA, in 2022.

He is currently the Motion Control and Planning Lead at Apptronik, Austin, TX, USA, where he specializes in navigation planning, dynamic locomotion, and contact-rich manipulation for humanoid robots.

\end{IEEEbiography}

\begin{IEEEbiography}[{\includegraphics[width=1in,height=1.25in,clip,keepaspectratio]{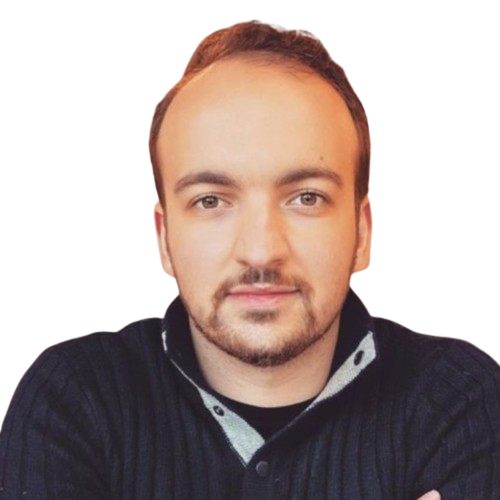}}]{Ardalan  Tajbakhsh}
(Student Member, IEEE) received the B.S.~degree in mechanical engineering with honors from the University of Illinois, Urbana-Champaign, IL, USA in 2018, and the M.S.~degree in mechanical engineering with concentration in robotics from Carnegie Mellon University, Pittsburgh, PA, USA, in 2020.

He is currently a PhD candidate at Carnegie Mellon University. His research interests include scalable motion planning and control for multi-agent systems under uncertain real-world conditions. 
\end{IEEEbiography}

\begin{IEEEbiography}[{\includegraphics[width=1in,height=1.25in,clip,keepaspectratio]{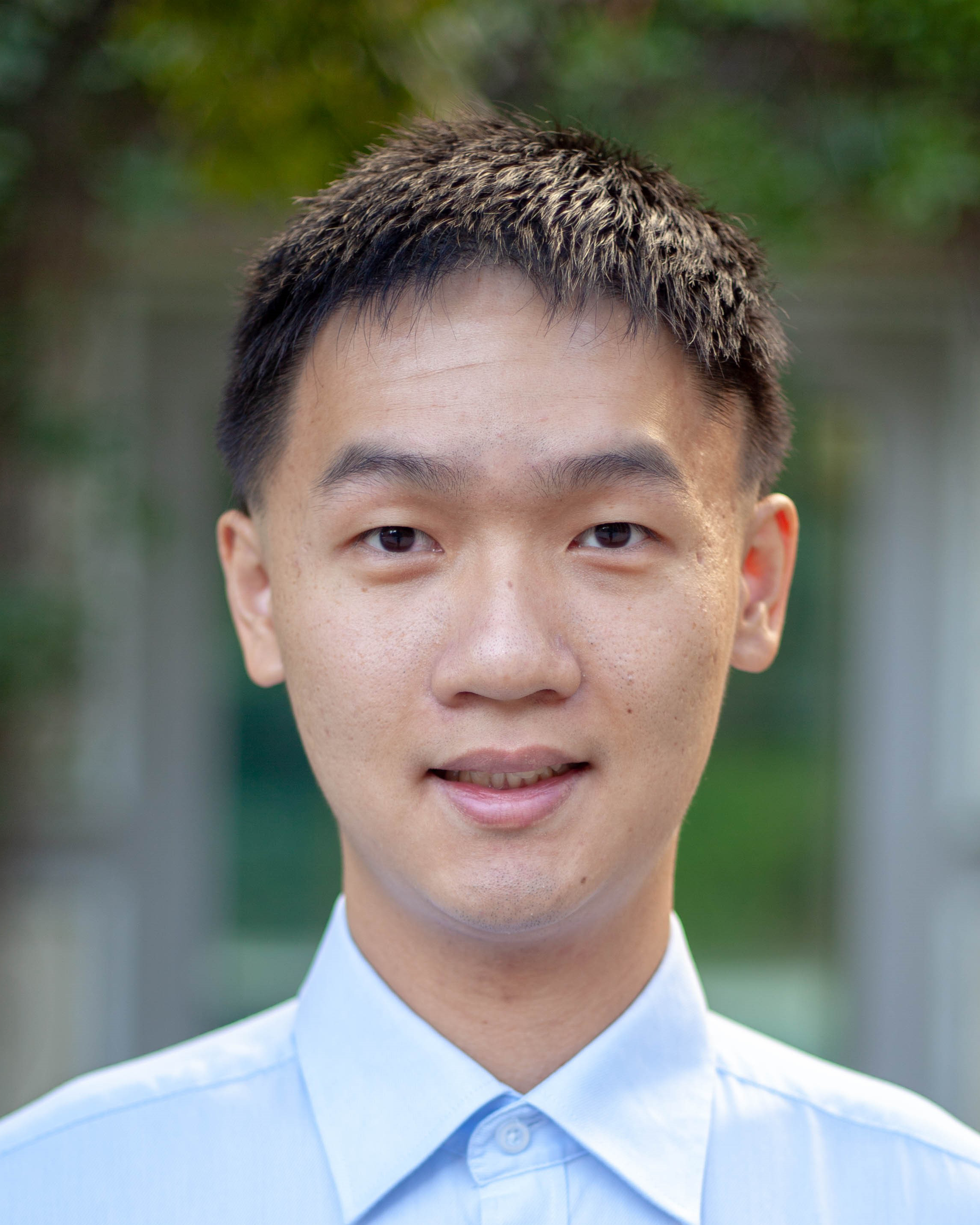}}]{Yanhao Yang}
(Student Member, IEEE) received the B.E.~degree in mechanical engineering from the South China University of Technology, Guangzhou, China in 2020, and the M.S.~degree in mechanical engineering from Carnegie Mellon University, Pittsburgh, PA, USA, in 2022.

He is currently a PhD student at Oregon State University, where he specializes in geometric motion planning and control for bio-inspired robots. 
\end{IEEEbiography}

\begin{IEEEbiography}[{\includegraphics[width=1in,height=1.25in,clip,keepaspectratio]{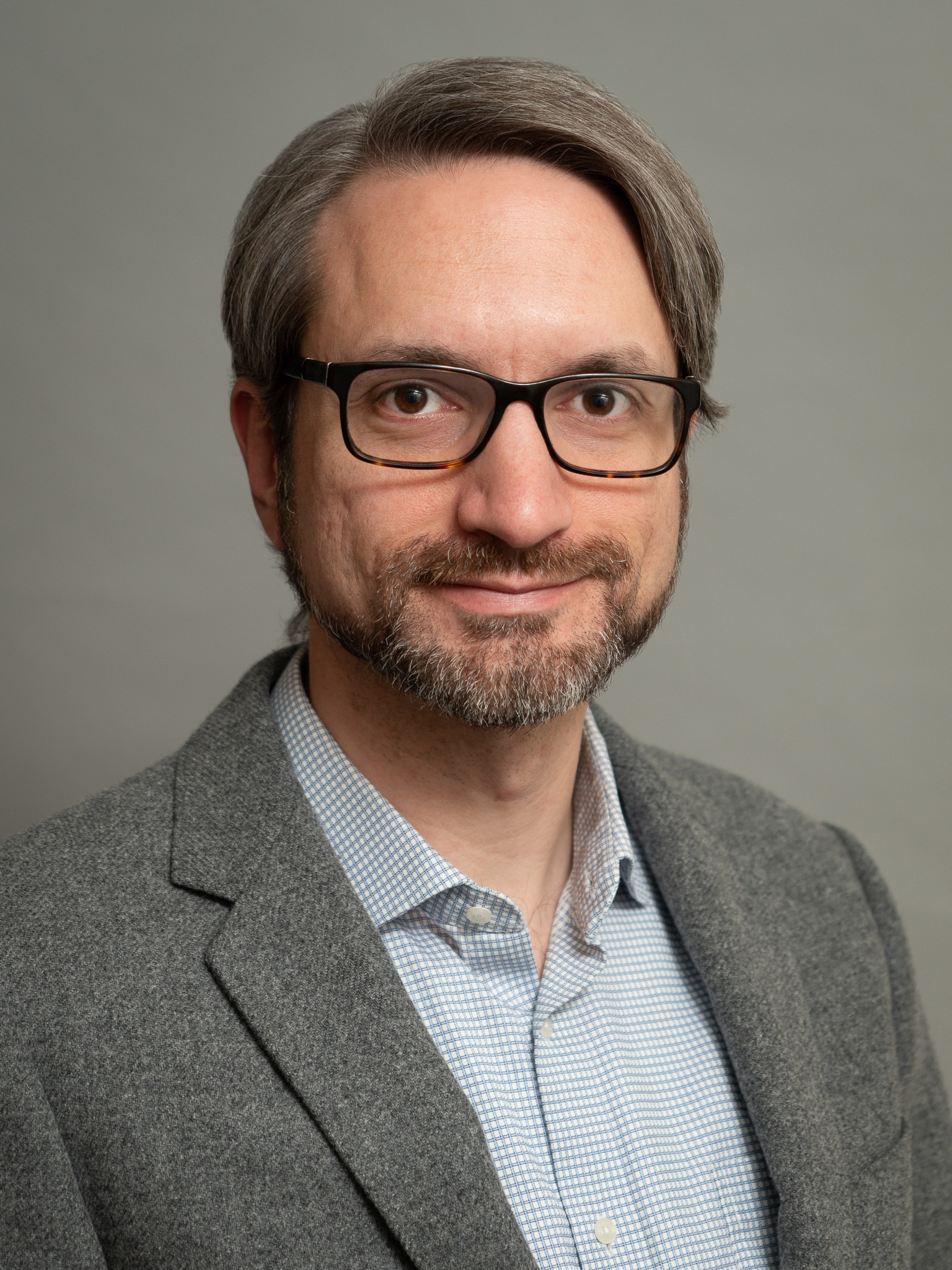}}]{Aaron M. Johnson} (S'06–M'14-SM'19)
received the B.S. degree in electrical and computer engineering from Carnegie Mellon
University, Pittsburgh, PA, USA, in 2008 and the Ph.D.
degree in electrical and systems engineering from the
University of Pennsylvania, Philadelphia, PA, USA, in
2014. 

He is currently an Associate Professor of Mechanical Engineering at Carnegie Mellon University,
with appointments in the Robotics Institute and Electrical and Computer Engineering Department. He was
previously a Postdoctoral Fellow at Carnegie Mellon
University and the University of Pennsylvania. 
His research interests include legged locomotion, hybrid dynamical systems, robust control, and bioinspired robotics. 
\end{IEEEbiography}

\balance

\end{document}